\documentclass[11pt]{article}
\usepackage[dvips]{graphicx}
\usepackage{amssymb,amsmath,color}
\usepackage{url,subcaption}

\usepackage[mathcal]{eucal}

\DeclareMathAlphabet\mathbfcal{OMS}{cmsy}{b}{n}

\usepackage[colorlinks=true, allcolors=blue, pagebackref]{hyperref}       
\renewcommand*{\backrefalt}[4]{%
    \ifcase #1 \footnotesize{(not cited)}%
    \or        \footnotesize{(cited on 	page~#2)}%
    \else      \footnotesize{(cited on pages~#2)}%
    \fi}

\newcommand{\BEAS}{\begin{eqnarray*}}
\newcommand{\EEAS}{\end{eqnarray*}}
\newcommand{\BEA}{\begin{eqnarray}}
\newcommand{\EEA}{\end{eqnarray}}
\newcommand{\BEQ}{\begin{equation}}
\newcommand{\EEQ}{\end{equation}}
\newcommand{\BIT}{\begin{itemize}}
\newcommand{\EIT}{\end{itemize}}
\newcommand{\BNUM}{\begin{enumerate}}
\newcommand{\ENUM}{\end{enumerate}}
\newcommand{\BA}{\begin{array}}
\newcommand{\EA}{\end{array}}

\newcommand{\tr}{\mathop{ \rm tr}}

\newcommand{\rb}{\mathbb{R}}

\newcommand{\BlackBox}{\rule{1.5ex}{1.5ex}}  
\newcommand{\lova}{Lov\'asz }

\newenvironment{proof}{\par\noindent{\bf Proof\ }}{\hfill\BlackBox\\[2mm]}
\newtheorem{lemma}{Lemma}

\newtheorem{proposition}{Proposition}

\newcommand{\mysec}[1]{Section~\ref{sec:#1}}
\newcommand{\eq}[1]{Eq.~(\ref{eq:#1})}
\newcommand{\myfig}[1]{Figure~\ref{fig:#1}}

\def \E{{\mathbb E}}
\def \P{{\mathbb P}}

\def \X{{\mathcal X}}
\def \Y{{\mathcal Y}}

\def \P{{\mathbb P}}

\usepackage{stackengine}
\usepackage{scalerel}
\usepackage{xcolor}

\title{A Convex Loss Function for Set Prediction \\ with Optimal Trade-offs Between Size and Conditional Coverage}

\author{Francis Bach \\
Inria - Ecole Normale Sup\'erieure\\
PSL Research University}

\date{\today}

\begin{document}
\maketitle

\begin{abstract}
We consider supervised learning problems in which set predictions provide explicit uncertainty estimates. Using Choquet integrals (a.k.a.~\lova extensions), we propose a convex loss function for non-decreasing subset-valued functions obtained as level sets of a real-valued function. This loss function allows optimal trade-offs between conditional probabilistic coverage and the ``size'' of the set, measured by a non-decreasing submodular function.  We also propose several extensions that mimic loss functions and criteria for binary classification with asymmetric losses, and show how to naturally obtain sets with optimized conditional coverage. We derive efficient optimization algorithms, either based on stochastic gradient descent or reweighted least-squares formulations, and illustrate our findings with a series of experiments on synthetic datasets for classification and regression tasks, showing improvements over approaches that aim for marginal coverage.
  \end{abstract}

\section{Introduction}

Uncertainty quantification is crucial in high-dimensional prediction problems. For supervised learning problems, among several frameworks, given a specific, previously unseen input, we consider outputting a subset of the output space with high probability of containing the random, unknown output, a problem we refer to as \emph{set prediction}. This is a classical problem in statistics and machine learning, which has been approached from several perspectives, ranging from parametric or non-parametric confidence sets~\cite{casella2024statistical,robert2007bayesian} to conformal prediction~\cite{vovk2005algorithmic,shafer2008tutorial,angelopoulos2021gentle}.

\paragraph{Trade-off between conditional coverage and size.}
Given a joint distribution over $\X \times \Y$, our goal is to consider set predictions $A: \X \to \mathcal{P}(\Y)$ (the set of all measurable subsets of $\Y$, where $\Y$ is a measurable set) that achieve, for all observed $x \in \X$, the optimal trade-off between a well-defined notion of ``size'' $V(A(x))$ and \emph{conditional miscoverage} $\P(Y \notin A(X)|X=x)$. We aim to achieve this through a convex loss formulation, where the empirical risk (average of loss function values on observations) is used for training, and expected risk is used for testing, with the usual assumptions in supervised learning that training and testing distributions are the same.

In this paper, we consider $V: \mathcal{P}(\Y) \to \rb$ such a measurement function of size, which we assume to be non-decreasing (i.e., $V(A) \leqslant V(B)$ if $A \subseteq B$), such that $V(\varnothing)=0$ (which implies that $V$ is non-negative), and $V(\Y)$ is finite.

Among the several ways to obtain a trade-off, the minimum expected size given conditional coverage, that is,
$$\inf_{A: \X \to \mathcal{P}(\Y)} \E[ V(A(X)) ] \ \mbox{ such that } \ \P( Y \notin A(X) | X ) \leqslant \alpha \ \mbox{ almost surely},$$
is natural in many applications. Like common supervised learning criteria, it
decomposes across every $x \in \X$, that is,  the optimal set $A(x)$ for a specific $x \in \X$ should be a minimizer of 
\BEQ
\label{eq:optA}
\inf_{B \subseteq \Y} V(B) \ \mbox{ such that }\ \P( Y \notin B | X=x) \leqslant \alpha.
\EEQ
However, this natural formulation poses several problems: (a) Because of the constraint, unbiased evaluation or estimation from data requires the knowledge of the conditional distribution, and (b) even if the conditional distribution of $Y$ given $X$ is known, this problem may be computationally hard to solve when $\Y$ is large (the computational intractability occurs even in the simplest case where $V$ is a probability measure, as this leads to the $\{0,1\}$-knapsack problem, which is known to be intractable~\cite{cormen2022introduction}).

\paragraph{Lagrangian formulation.}
 In order to tackle the computational problem (b), we consider the Lagrangian relaxation of \eq{optA}, that is,
 \BEQ
\label{eq:optAlag}
\sup_{\lambda \geqslant 0}\  \inf_{B \subseteq \Y} \ V(B) + \lambda \big[  \P( Y \notin B | X = x) -  \alpha \big],
\EEQ
which is in general not tight, that is, not equal to the optimal value in \eq{optA} (and for which the optimal $\lambda$ is typically different for every $x$). For simple functions $V$ such as positive additive measures, the problem in \eq{optAlag} can be easily solved by sorting normalized densities~\cite{cormen2022introduction}, and is known as the \emph{fractional relaxation} and can be interpreted through randomized predictions (see \mysec{convexloss}).

In order to obtain a generic tractable formulation for \eq{optAlag} beyond positive additive measures, we restrict~$V$ to a known class of set functions. In this paper, we consider \emph{submodular} functions, which satisfy
$$
\forall A,B \subseteq \Y, \ V(A \cap B) + V(A \cup B) \leqslant V(A) + V(B),
$$
for which the computational problem in \eq{optAlag} is tractable.
See~\cite{fujishige2005submodular,bach2013learning} for an introduction to submodular functions, and examples such as additive measures (for which there is equality above and which are referred to as modular functions) or set-covers (that can discourage sets with too many connected components, see \mysec{regression}), and a review of main properties in \mysec{lova}.

A key aspect is that the Lagrange multiplier $\lambda$ has to depend on $x$; otherwise, we only get marginal coverage. Indeed, with the same Lagrange multiplier $\lambda$ for all $x \in \X$, we obtain the problem
\BEAS
& & \sup_{\lambda \geqslant 0} \inf_{A: \X \to \mathcal{P}(\Y)}
\E\Big[
V(A(X)) + \lambda \big[ \P(Y \!\notin\! A(X) | X ) - \alpha  \big]
\Big] 
\\
& = & 
\sup_{\lambda \geqslant 0} \inf_{A: \X \to \mathcal{P}(\Y)}
\E\big[
V(A(X))\big] +\lambda \big[ \P(Y \!\notin\! A(X)) -  \alpha  \big],
\EEAS
which exactly leads to the Lagrangian relaxation of the optimal \emph{marginal} coverage problem
$$
\inf_{ A: \X \to \mathcal{P}(\Y)} \E[ V(A(X)) ] \ \mbox{ such that } \ \P(Y \notin A(X))  \leqslant \alpha,
$$
as studied in~\cite{braun2025minimum} for certain shapes of sets (e.g., $\ell_p$-norm balls) and by~\cite{sadinle2019least} in the discrete case. Our goal is to go beyond and explicitly tackle conditional coverage guarantees, which require that the Lagrange multiplier $\lambda$ depends on $x \in \X$. Moreover, beyond the coverage issue, we still need to efficiently parameterize $A: \X \to \mathcal{P}(\Y)$.

This paper is based on three main ideas:
\BIT
\item[(1)] Learn all set functions for all values of the trade-off parameter $\lambda$ through the level sets of a real-valued function $g: \X \times \Y \to \rb$ as 
$$A(\lambda,x) = \big\{ y \in \Y, \ g(x,y) \geqslant -\lambda\big \}.$$
Parameterizing sets through level sets has a long history in signal and image processing~\cite{osher1988fronts,osher2003geometric}, and here leads to a non-decreasing (in $\lambda$) function $A(\lambda,x)$ based on a real-valued function $g$.

\item[(2)] Realize that the integral with respect to $\lambda$ of the natural loss $V(A(X)) + \lambda \P(Y \notin A(X))$ leads to a convex function in $g$ which depends on the Choquet integral (a.k.a.~\lova extension) of the function $V: \Y \to \rb$, which is a classical concept in submodular analysis. This naturally leads to a convex learning formulation for finite data, with the added benefit of allowing the computation of all fractional solutions to all conditional miscoverage problems (for any $x \in \X$ and any $\alpha \in [0,1]$). For additive measures, this leads to a novel non-standard quadratic loss function. 

\item[(3)] Leverage the availability of solutions for all $\lambda$ to still obtain good conditional coverage, by naturally defining a function $\lambda^\ast(\alpha,x)$ for the value of $\lambda$ for a given input $x \in \X$ and some arbitrary miscoverage level $\alpha \in (0,1)$.
\EIT

\subsection{Contributions}
We make the following contributions:
\BIT
\item We propose in \mysec{convexloss} a convex loss function for non-decreasing subset-valued functions obtained as level sets of a real-valued function. This loss function allows optimal trade-offs between conditional probabilistic coverage and the ``size'' of the set, measured by a non-decreasing submodular function. We provide examples in \mysec{examples} that cover both discrete and continuous sets $\Y$.
\item We propose in \mysec{otherlosses} several extensions mimicking loss functions and criteria for binary classification with asymmetric losses. These loss functions can be used to evaluate any non-decreasing subset-valued function, whereas the one proposed in \mysec{convexloss} can be used for both learning and evaluation.
\item We show in \mysec{conformal} how to naturally obtain sets with conditional coverage and how to use conformal prediction to at least ensure a posteriori marginal coverage.
\item We derive efficient optimization algorithms in \mysec{guarantees}, either via stochastic gradient descent or via reweighted least-squares formulations.

\item We illustrate our findings with a series of experiments in \mysec{experiments} on synthetic datasets, on classification and regression tasks. We compare our estimators based on our new loss functions to baselines based on the pinball loss~\cite{bookQuantile}, baselines based on implicit or explicit density estimation, or baselines derived from a simpler Lagrangian formulation (input-independent Lagrange multiplier).
\EIT

\subsection{Related work}
In this section, we describe existing convex formulations for set predictions or the related estimation of quantiles of a one-dimensional probability distribution.

\paragraph{Pinball loss.}
Minimizing with respect to $g: \X  \to \rb$, the following risk
$$
\E \big[
\alpha (g(X)-Y)_+ + (1-\alpha) ( Y - g(X))_+
\big]
$$
leads to an optimal prediction such that for all $x$ (see \cite{bookQuantile,romano2019conformalized})
$$
0 = \E\big[ \alpha  1_{g(X)-Y \geqslant 0}  -  (1-\alpha)  1_{g(X)-Y < 0} \big| X = x\big]
= \P( Y \leqslant g(X)| X=x) - (1-\alpha),
$$
which leads to a conditional coverage for the set $\{ Y \leqslant g(X)\}$, if (1) we can optimize over all measurable functions $g$, and (2) we have enough data for this estimation (requirements that apply also to our own framework). This allows us to learn any intervals $  [ g(X), h(X) ]$ by learning two quantiles (this is precisely the ``interval loss''~\cite{winkler1972decision,gneiting2007strictly}). Multivariate extensions exist but are not as straightforward~\cite{carlier2016vector}, but both with univariate and multivariate outputs, there is no notion of optimal size. Extensions that take into account a notion of size, such as the volume, have been developed for special sets, but for marginal coverage~\cite{sadinle2019least,pouplin2024relaxed,braun2025minimum}. Also, it is worth noting the similarity (in spirit) between our integration over all trade-off parameters $\lambda$ and the integration over all $\alpha$ above, which leads to the continuous ranked probability score~\cite{matheson1976scoring}.

Another link with the pinball loss would be the possibility of learning the threshold function $\lambda: \X \to \rb$ through the following loss function
$$
\E  \big[ \alpha ( g(X,Y)+\lambda(X) ) _+ + (1-\alpha) ( -g(X,Y)-\lambda(X) ) _+\big],
$$
that would ensure, if the optimization with respect to $\lambda$ is done ``correctly'' (with enough data and an expressive enough model), that we have the correct conditional coverage for all sets $\{ g(x,\cdot) + \lambda(x) \geqslant 0\}$, but with no notion of size.

\paragraph{Set predictions.}
This has been formulated as a structured prediction problem, with no notion of size and no focus on conditional coverage guarantees~(see \cite{cabannnes2020structured} and references therein).

\paragraph{Density estimation.} One simple way is to run a conditional density estimator, e.g., using square loss or maximum likelihood, and then solve \eq{optA} or its relaxation using the estimate. However, this is relying on a task (learning the whole density) which is harder (in particular for $\Y$ continuous, or $\Y$ discrete with large cardinality) than simply providing a set with high coverage.
Moreover, given the conditional density estimate, the computational problem remains to be solved efficiently.

\paragraph{Conditional coverage for set prediction.} Our work proposes convex cost functions that will lead to conditional coverage in the infinite sample limit, and thus leads to good candidates for procedures like conformal prediction that can provably get marginal coverage~\cite{vovk2005algorithmic,shafer2008tutorial,angelopoulos2021gentle}. We are thus not addressing the known difficulties of provably obtaining conditional coverage~\cite{lei2014distribution,foygel2021limits}.

  \section{Convex loss functions for subset-valued functions}
  \label{sec:convexloss}
 We consider the problem of predicting $y \in \Y$ from $x\in \X$, and a subset-valued function $A: \rb \times \X \to \mathcal{P}(\Y)$ (the set of all measurable subsets of $\Y$), which is non-decreasing in the first argument, that is, for all $x \in \X, \lambda \in \rb ,\lambda' \in \rb$, $$ \lambda < \lambda' \ \Rightarrow \ A(\lambda,x) \subseteq A(\lambda',x).$$
 We define a loss function $\ell_\lambda: \Y \times \mathcal{P}(\Y) \to \rb$,
 $$\ell_
 \lambda(y,B)=V(B) -V(\Y)+ \lambda 1_{y\notin B},
 $$
 which explicitly trades off size and miscoverage via a parameter $\lambda>0$ (we add the term $-V(\Y)$ to ensure the loss is zero for $B=\Y$, which will be needed later).
Our goal is to estimate $A(\lambda,\cdot)$, so that, for a fixed~$\lambda \geqslant 0$, for expectations with respect to the test distribution,
 $$\E\big[\ell_
 \lambda(Y,A(\lambda,X))\big]	 $$
 is minimized. 
 Using standard arguments from learning theory~\cite{devroye2013probabilistic,bach2024learning}, all solutions of the non-constrained problem (i.e., that are not constrained to be monotonic) satisfy
 \BEQ
 \label{eq:All} \forall \lambda \geqslant 0, \ \forall x \in \X , \ 
 A(\lambda,x) \in \operatorname*{arg\,min}_{B \subseteq \Y} \ \E\big[ \ell_\lambda(Y,B)|X=x],
 \EEQ
 which turns out to be non-decreasing in $\lambda$ as soon as $V$ is submodular (see \cite{topkis1978minimizing} or Prop.~8.1 in \cite{bach2013learning}). 
 
 In order to  learn $A(\lambda,\cdot)$ for all $\lambda \geqslant 0$, a natural criterion to consider  is 
 \BEQ
 \label{eq:natural}
 \int_0^\infty \E\big[\ell_
 \lambda(Y,A(\lambda,X))\big] d\lambda,
 \EEQ
 which indeed leads to the solution in \eq{All}. Other weighting functions (depending on $\lambda$) could be considered, but, as seen below, the chosen one (no weights) leads to a nice closed-form formula (see the end of \mysec{finalloss} for alternatives).

\paragraph{Randomized predictions.} Throughout the paper, since the problem in \eq{optA} can only be solved efficiently and reliably with randomized predictions, we will consider both deterministic predictions where a single subset is output, but also randomized predictions where \emph{two} sets can be output with certain probabilities. Then, expectations with respect to this extra randomness are considered before expectations with respect to the data.

 \subsection{Integrated loss functions for additive measures}
A classical way of encoding non-decreasing set functions is through sup-level sets, that is, 
\BEQ
\label{eq:Alev}
A(\lambda,x) = \big\{ y \in \Y, \ g(x,y) \geqslant -\lambda\} = \{ g(x,\cdot) \geqslant -\lambda \big\},
\EEQ
where $g:\X \times \Y \to \rb$~\cite{osher1988fronts,osher2003geometric}. In order to account for all potential trade-offs between size and miscoverage, as defined in \eq{natural}, we consider the integrated loss function 
$$
\int_0^{+\infty}   \ell_
 \lambda(y,A(\lambda,x))  d \lambda =\int_0^{+\infty}   \big( V(\{ g(x,\cdot) \geqslant -\lambda\}) - V(\Y)+  \lambda 1_{\{ g(x,y) < -\lambda\}} \big) d\lambda,
 $$
(note that the term $- V(\Y)$ implies that the summand is integrable when $\lambda$ tends to $+
\infty$). 
The following lemma shows that when $V$ is a finite additive measure, there is a closed form for the integrated loss taken at $A$ defined in \eq{Alev}. To obtain simpler formulas, we assume that the function we take level sets of is non-positive and that the threshold $\lambda$ is non-negative, which we will relax later.
\begin{lemma}
\label{lemma:ext}
If $V$ is a finite additive measure, then for any $h:\Y\to \rb_-$ and additive measure $Q$ on $\Y$, we have (with $B^{\sf c}$ denoting the complement of the set $B$):
$$
\int_0^{+\infty} \Big( V(\{ h \geqslant -\lambda\}) - V(\Y) + \lambda Q( \{ h \geqslant -\lambda\}^{ \sf c}) \Big)d\lambda = \int_\Y h(z) dV(z) + \frac{1}{2}\int_\Y h(z)^2 dQ(z).$$
\end{lemma}
\begin{proof}
This is exactly the ``layer cake representation''~\cite[Section 1.13]{lieb2001analysis}: for any differentiable function $\varphi:\rb_+ \to \rb$ such that $\varphi(0)=0$ and any additive measure $Q$, $$\int_0^{+\infty} 
 \varphi'(\lambda) Q(\{h <  -\lambda\})d \lambda=\int_\Y \varphi(-h(z)) dQ(z),$$
 (which can be shown by starting with $Q$ being a Dirac measure\footnote{When $Q$ is a Dirac measure at some $u\in \Y$, then this corresponds to 
 $\int_0^{+\infty} 
 \varphi'(\lambda) 1_{h(u) <  -\lambda} d\lambda = \int_0^{-h(u)} \varphi'(\lambda)d\lambda= \varphi(-h(u))$.} and extending by averaging)
 which is applied to $\varphi(\lambda) =\lambda$ and  $\varphi(\lambda) =   \frac{1}{2} \lambda^2 $, and to the measures $V$ and $Q$. 
\end{proof}
This is the simplest formulation we consider in this paper for additive measures, leading to a new type of quadratic loss function. The layer-cake representation can be extended beyond additive measures, and is also then referred to as the co-area formula~\cite{chambolle2010introduction}, as done in the \mysec{lova} below.

 \subsection{\lova extension / Choquet integral}
 \label{sec:lova}
 A set-function $V: \mathcal{P}(\Y) \to \rb$ can be identified to a function $v$ on measurable functions $f: \Y \to \{0,1\}$ through the relationship
 $$
 \forall A \in \mathcal{P}(\Y), \ 
 v(1_A)=V(A),
 $$
 where $1_A: \Y \to \{0,1\}$ is the indicator function of the measurable set $A$.
 The function $v$ can be extended to functions $f: \X \to \rb_{\textcolor{red}{+}}$, through the Choquet integral formula~\cite{denneberg1994non}:
\BEQ
\label{eq:lovadef} v(f) = \int_{0}^{+\infty} V(\{f \geqslant t\}) dt
= \int_{0}^{+\infty} V(\{y \in \Y, \ f(y) \geqslant t\}) dt.
\EEQ
 If $f=1_A$ for some $A \subseteq \Y$, then
 we have $v(1_A) = \int_0^{+\infty} \big( V(A)1_{[0,1]}(t) dt + V(\varnothing)1_{(1,+\infty)}(t) \big) dt
 = V(A)$, and 
  we indeed recover $V(A)$ (we assumed $V(\varnothing)=0$). To define it for functions $f: \Y \to \rb$ (i.e., with potentially negative values), we cannot simply use $ \int_{-\infty}^{+\infty} V(\{f \geqslant t\}) dt$ as the integral is not convergent at $-\infty$ because the integrand is converging to $V(\Y)$. A simple modification leads to a positively homogeneous function, defined for all $f:\Y \to \rb$ (see \cite{bach2013learning}),
 \BEQ
\label{eq:choquet}
 v(f) = \int_{0}^{+\infty} V(\{f \geqslant t\}) dt + \int_{-\infty}^0 \big[ V(\{f \geqslant t\}) - V(\Y) \big] dt.
 \EEQ
 We refer to this extension as the \lova extension. 
 In particular, when $f$ is non-positive (which we need to extend Lemma~\ref{lemma:ext}), that is, $f:\Y \to \rb_-$, we get
 $$
 v(f) =  \int_{0}^{+\infty} \big[ V(\{f \geqslant -t\}) - V(\Y) \big] dt.
 $$
 Moreover, if $V$ is a Dirac at $u \in \Y$, we get from \eq{choquet}, for any $f: \Y \to \rb$, $v(f) =
 \int_{0}^{+\infty} 1_{f(u) \geqslant t} dt + \int_{-\infty}^0  [ 1_{f(u) \geqslant t} - 1]    dt = f(u). $
 This extends to additive measures, for which we get $v(f) = \int_\Y f(z) dV(z)$, that is, exactly the earlier layer cake representation used in the proof of Lemma~\ref{lemma:ext} (see more examples in \mysec{examples}).
 
 Overall, this leads to the following extension of Lemma~\ref{lemma:ext}.

 \begin{lemma}
\label{lemma:extgen}
For any $h:\Y\to \rb_-$ and additive measure $Q$, we have:
\BEQ
\label{eq:costlemma2}
\int_0^{+\infty} \big( V(\{ h \geqslant -\lambda\}) - V(\Y) + \lambda Q( \{ h \geqslant -\lambda\}^{\sf c})\big) d\lambda = v(h) + \frac{1}{2}\int_\Y h(z)^2 dQ(z).
\EEQ
\end{lemma}
 This shows a precise link between the \lova extension and the integrated loss, \emph{without any assumptions} beyond measurability. With more assumptions, we get more properties, as we now show.

 \paragraph{Convexity and submodularity.} There is a strong link between properties of a set function: $V \to \mathcal{P}(\Y)$ and its \lova extension $v$: $V$ is submodular if and only if its \lova extension $v$ is convex, with several different proofs~\cite{lovasz1983submodular,bach2013learning,bach2019submodular}.
 Moreover, assuming from now on that $V$ is submodular, we can compute subgradients and explicit links between minimizers of an optimization problem in $v$, and a sequence of problems in $V$.

\paragraph{Computing values and subgradients of $v$.}
The function $v$ is convex, 1-homogeneous, and has a full domain. It can thus be represented as the supremum of linear functions over a bounded set of measures, that is, $v(f) = \sup_{ \mu \in B(V) } \int_\Y f(y) d\mu(y)$, where $B(V)$ is called the ``base polytope'' when $\Y$ is finite \cite[Section 4]{bach2013learning}, and the ``core'' in general~\cite[Chapter 10]{denneberg1994non}, and defined as
$$
B(V) = \big\{ \mu\mbox{ measure on }\Y ,  \ \mu(\Y) = V(\Y), \ \forall A \in \mathcal{P}(\Y), \mu(A) \leqslant V(A) \big\}.
$$
When $V$ is non-decreasing, then $B(V)$ happens to be composed only of nonnegative measures. Maximizers $\mu$ for a given~$f$ (that are subgradients of $v$ at $f$) can be obtained from level sets of $f$ through a so-called ``greedy algorithm,'' which, for finite sets~$\Y$, sorts the values of $f$ and computes values of $V$ as level-sets of $f$. Note that this possibility of computing subgradients leads to polynomial time algorithms for submodular function minimization~\cite{lovasz1983submodular}. In the general case, when $f $ takes $m$ values $w_1> \cdots > w_m$ ordered in strictly decreasing order on sets $B_1,\dots,B_m$ that form a partition of $\Y$, the minimizers satisfy $\mu(B_i) = V(B_1 \cup \cdots \cup B_i) - V(B_i)$ for all $i \in \{1,\dots,m\}$. For our examples in \mysec{examples}, maximizers will have explicit formulas.

\paragraph{Submodular function minimization.}
Since $v$ is an extension of $V$, we
have
$$ \inf_{B \subseteq \Y} V(B) = \inf_{f: \Y \to \textcolor{red}{\{}0,1 \textcolor{red}{\}}} v(f).$$
When $V$ is submodular, this happens to be equal to \BEQ
\label{eq:SFM}
\inf_{f: \Y \to  \textcolor{red}{[}0,1 \textcolor{red}{]}} v(f),
\EEQ
 which is now a convex optimization problem (see, e.g.,~\cite{bach2013learning}).

Given a solution $f$ of \eq{SFM}, with real values in $[0,1]$, the randomized prediction rule defined by $\{ y \in \Y, \ f(y) \geqslant t\}$, with $t$ uniformly distributed in $[0,1]$, leads to the optimal value (by definition of the \lova extension). It turns out that a deterministic minimizer can be found from its sup-level sets.

Constrained optimization, however, even with the simplest modular constraint, cannot be solved in polynomial time~\cite{cormen2022introduction}. Indeed, for a measure $\mu$ on $\Y$, the problem 
$$ \inf_{B \subseteq \Y} V(B)  \mbox{ such that } \mu(B) \leqslant c \ \ \ = \ \ \ \inf_{f: \Y \to \textcolor{red}{\{}0,1\textcolor{red}{\}}} v(f)  \mbox{ such that } \int_\Y f(y) d\mu(y) \leqslant c$$
can be strictly greater than 
\BEQ
\label{eq:SFMc}
\inf_{f: \Y \to \textcolor{red}{[}0,1\textcolor{red}{]} } v(f)  \mbox{ such that } \int_\Y f(y) d\mu(y) \leqslant c,
\EEQ
except for a small number of values of $c$.
From a solution $f$ of \eq{SFMc}, the randomized prediction rule is optimal among all randomized prediction rules, with $V(B)$ and $\mu(B) $ replaced by expectations over the randomness of the rule. However, the function $f$ with real values cannot be used to obtain an optimal deterministic solution (as opposed to the unconstrained case).


\paragraph{Links between optimization problems.} Our loss function based on Lemmas~\ref{lemma:ext} and~\ref{lemma:extgen} leads, once specialized to a single $x \in \X$, to convex optimization problem of the form
\BEQ
\label{eq:vp}
\inf_{ f: \Y \to \rb} v(f) + \int_\Y \varphi(f(z),z) dQ(z), 
\EEQ
where for each $z \in \Y$, $\varphi(\cdot,z)$ is convex and $Q$ a positive additive measure (e.g., in \eq{costlemma2} from Lemma~\ref{lemma:extgen}, $\varphi(f(z),z) = \frac{1}{2} f(z)^2$). It turns out~\cite[Section 8]{bach2013learning} that its solutions are related to a
sequence of set-optimization problems, for $\lambda \in \rb$,
\BEQ
\label{eq:Vp}
\inf_{B \subseteq \Y} \Big\{ V(B) + \int_B \varphi'(\lambda,z) dQ(z)\Big\}
= \inf_{ f: \Y \to \{0,1\}} \Big\{ v(f) + \int_\Y  \varphi'(\lambda,z)f(z)  dQ(z) \Big\}.
\EEQ
The following lemma (see proof in \cite[Prop.~8.5]{bach2013learning}) shows that the suboptimality gap for \eq{vp} is the integral of the submodularity gaps for \eq{Vp} over $\lambda$, with the candidate sets that are sup-level sets of~$f$.  This is much stronger than the definition of the \lova extension through sup-level sets. Primal-dual guarantees also exist~\cite[Prop.~8.5]{bach2013learning}.
\begin{lemma}
\label{lemma:calibration}
Assume that for each $z\in \Y$, $\varphi(\cdot,z) $ is strictly convex and differentiable on $\rb$ and such that its Fenchel conjugate has full domain. Assume $Q$ is a positive additive measure.
    For any $f:\Y \to \rb$, we have:
    \BEAS
    & & v(f) + \int_\Y \varphi(f(z),z) dQ(z) - 
    \inf_{ g: \Y \to \rb} \Big\{ v(g) + \int_\Y \varphi(g(z),z) dQ(z) \Big\}
    \\
    & = & 
    \int_{-\infty}^{+\infty}
    \bigg[
    V(\{f \geqslant \lambda\}) + \int_{\{f \geqslant \lambda\}} \varphi'(\lambda,z) dQ(z) -
\inf_{B \subseteq \Y} \Big\{ V(B) + \int_B \varphi'(\lambda,z) dQ(z)  \Big\}  
    \bigg] d\lambda.
    \EEAS
\end{lemma}

\subsection{Final loss function}
\label{sec:finalloss}
Given properties of the \lova extension described in \mysec{lova}, we propose the following loss function, for $y \in \Y$ and $f: \Y \to \rb$,
\BEQ
\label{eq:loss}
\ell: (y,f) \mapsto v(f) + \frac{1}{2} f(y)^2,
\EEQ
without any restrictions on the negativity of $f$.
The following proposition is the key contribution of this paper and a direct consequence of Lemma~\ref{lemma:calibration}.

\begin{proposition}
\label{prop:main}
 For any function $g: \X \times \Y \to \rb$, and the loss defined in \eq{loss},
\BEAS
 & & \E[ \ell( Y, g(X,\cdot)) ] 
- \inf_{h: \X \times \Y \to \rb}\E[ \ell( Y, h(X,\cdot)) ] \\
& = & 
\E \bigg[
\int_{-\infty}^{+\infty}\Big(
V(\{ g(X,\cdot) \geqslant -\lambda\}) +  
\lambda \P(  \{ g(X,\cdot) \geqslant -\lambda\}^{\sf c}| X)  - \inf_{B \subseteq \Y} \big\{ V(B) + \lambda \P( Y \notin B | X) \big\}
\Big) d\lambda \bigg].
\EEAS   
\end{proposition}

It exactly says that a minimizer $g^\ast$ of $\E[ \ell( Y, g(X,\cdot)) ] $ (what a supervised learning algorithm aims to do), will lead to minimizer of the optimal \emph{conditional} coverage problem $\inf_{B \subseteq \Y} \big\{ V(B) + \lambda \P( Y \notin B | X=x)\big\}$ for all $x \in \X$ and $\lambda \in \rb$, by selecting 
$$A(x) = \{ g(x,\cdot) \geqslant -\lambda\} = \{ y \in \Y, \ g(x,y) \geqslant -\lambda \}.$$
In other words, we defined a proper scoring rule~\cite{gneiting2007strictly} for estimating all size-optimal conditional coverage sets. Moreover, Prop.~\ref{prop:main} shows that there is even a ``calibration function'' relating the excess risk of our convex loss to the excess risks of all problems in $\lambda$, as for convex surrogates for binary classification~\cite{zhang2004statistical,bartlett2006convexity}.
  
Note that when $\lambda<0$, the minimizer is $B = \varnothing$, and there is no contribution from this part.

\paragraph{Alternative weighting functions.}
Other weighting functions between $V(B)$ and $\P( Y \notin B | X=x)$ than (up to constants) $\int_0^{+\infty}  [ V(B) + \lambda \P( Y \notin B | X=x)  ] d\lambda$, could be used to provide different trade-offs, such as $\int_0^{+\infty}  [ a(\lambda) V(B) + \lambda b(\lambda) \P( Y \notin B | X=x)  ] d\lambda$ for functions $a$ and $b$. This would lead to non-quadratic loss functions, but we focus on the simpler case that leads to quadratic loss functions.

\section{Examples of submodular functions}
\label{sec:examples}

 Any example of non-decreasing submodular functions from~\cite{fujishige2005submodular,bach2013learning} can be used. The two classes we will consider in this paper are:

\BIT
\item \textbf{Additive non-negative measures} (i.e., non-decreasing modular functions): $V(A) = \int_A dM(y)=M(A)$ for some non-negative finite measure $M$ on $\Y$. Then, the \lova extension is $v(g) = \int_\Y g(y) dM(y)$, and the risk based on the loss function in \eq{loss}
is $$\mathcal{R}(g) = \E \big[ \ell(Y,g(X,\cdot))] = \E \Big[ \int_\Y g(X,z) dM(z) + \frac{1}{2} g(X,Y)^2 \Big],$$ with an optimal function $g^\ast(x,y) = - \big( \frac{dp(y|x)}{dM(y)} \big)^{-1}$.

In terms of gradient for the loss at an observation $(x,y)$, we can obtain an unbiased one by simply sampling $z$ from $M$ and taking the gradient of $g(x,z) + \frac{1}{2}g(x,y)^2$.

It is a form of square loss, but different from the standard one used in least-squares regression. We could also consider concave functions of such functions, that is, $\varphi(M(A))$.

\item \textbf{Set-covers}: Given a function $S: \Y \to \mathcal{P}(\mathcal{Z})$, then $V(A) = M\big( \bigcup_{y \in A} S(y) \big)$ for $M$ a non-negative measure on a set $\mathcal{Z}$ (in most cases, $\mathcal{Z} = \Y$), is submodular. We then have: $$
v(g) = \int_\mathcal{Z} \Big\{ \sup_{S(y) \ni z } g(y) \Big\}  dM(z).
$$
In terms of gradient, we can obtain an unbiased one by simply sampling $z$ from $M$ and taking the gradient of $  \sup_{S(t) \ni z } g(t)  + \frac{1}{2}g(x,y)^2$, which requires to solve a maximization problem, which we assume solvable in this paper (e.g., in low dimensions by grid search).
\EIT

 We now precisely describe the classic examples that we will consider in our experiments. For each of them, we will define as well a positive additive measure $M$ such that $V-M$ is non-negative, and $V(\Y)= M(\Y)$.

\subsection{Finite sets with cardinality-based functions}
We consider in this example a finite set $\Y$ with $k$ elements, which we identify to $\Y = \{1,\dots,k\}$, and consider 
$$V(A)= \varphi(|A|),$$
where $\varphi$ is a non-decreasing concave function and $|A|$ the cardinality of $A$. We can then parameterize a function from $\X \times \Y$ to $\rb$ as $k$ functions $g_1,\dots,g_k: \X \to \rb$, or $g: \X \to \rb^k$. This provides new loss functions for multicategory classification. The associated measure $M$ is $M(A) = |A| \varphi(k) / k$.

\paragraph{Cardinality.} For $V(A)=\frac{1}{k} |A|$, we have $v(f) = f^\top 1_k / k$, and  the loss function we consider is
$$\ell(y,g(x)) = \frac{1}{k} \sum_{i=1}^k g_i(x) + \frac{1}{2} g_y(x)^2,$$
which is quadratic in $g(x)$, but different from the usual quadratic loss $\sum_{i=1}^k ( 1_{y=i} - g_i(x))^2$. Note that when~$k$ is large, there is an unbiased estimate by sampling $i$ uniformly on $\{1,\dots,k\}$ and taking $g_i(x)$ instead of $\frac{1}{k} \sum_{j=1}^k g_j(x)$, which does not require access to the whole vector $g(x)$ (as opposed to the softmax loss).

The optimal function is then $  g^\ast_i(x) = - \frac{1}{k} \big(\P(Y=i|X=x)\big)^{-1}$, while it is $g^\ast_i(x) = \P(Y=i|X=x)$ for the regular square loss.

\paragraph{General concave functions.} If we consider $V(A) = \varphi(|A|)$, for $\varphi$ concave such that $\varphi(0)=0$, then the
loss function can be computed from the order statistics of $g(x)$ as follows (see \cite[Section 6.1]{bach2013learning})  
$$\ell(y,g(x)) = \sum_{i=1}^k ( \varphi(i)-\varphi(i-1)) g_{\sigma(i)}(x)+ \frac{1}{2} g_y(x)^2,$$
for any ($x$-dependent) bijection $\sigma: \{1,\dots,k\} \to \Y$ such that $g_{\sigma(1)}(x) \geqslant \cdots \geqslant g_{\sigma(k)}(x)$. It can be rewritten as
$$\ell(y,g(x)) = (\varphi(k) -\varphi(k-1) )\sum_{i=1}^k g_i(x) + \sum_{r=1}^k ( 2 \varphi(r) - \varphi(r\!+\!1)-\varphi(r\!-\! 1)) \sum_{i=1}^r g_{\sigma(i)}(x),
$$
in terms of non-negative linear combinations of $\sum_{i=1}^r g_{\sigma(i)}(x)$, the sums of the $r$ largest components of~$g(x)$ (a classical convex function of $g(x)$~\cite{boyd2004convex}). See \mysec{algorithms} for a reweighted least-squares formulation.

  The optimal function can be obtained
for a fixed $x \in \X$ by sorting the vector of posterior probabilities of $Y$ given $X=x$, that is, $\pi_{\sigma(1)}(x) \geqslant \cdots \geqslant \pi_{\sigma(k)}(x)$, where $\sigma: \{1,\dots,k\} \to \Y$ is a bijection and $\pi_y(x) = \P(Y=y|X=x)$, then leading to the minimization of
$$
\frac{1}{2}\sum_{i=1}^k \pi_{\sigma(i)}(x)  f_{\sigma(i)}^2 + \sum_{i=1}^k ( \varphi(i)-\varphi(i-1)) f_{\sigma(i)}
= \frac{1}{2}\sum_{i=1}^k \pi_{\sigma(i)}(x) 
\Big[
f_{\sigma(i)} + \frac{\varphi(i)-\varphi(i-1)}{\pi_{\sigma(i)}(x) }
\Big]^2 + \mbox{ cst} ,
$$
subject to the constraint that $f_{\sigma(1)} \geqslant \cdots \geqslant f_{\sigma(k)}$, which can be solved by isotonic regression in time $O(k)$ by the pool-adjacent-violators algorithm~\cite{best1990active}. If $\varphi$ is linear then the solution is exactly $g^\ast_{\sigma(i)}(x) = -\frac{\varphi(i)-\varphi(i-1)}{\pi_{\sigma(i)}(x) }$, otherwise, there is pooling of components together.

An interesting subcase is $\varphi(|A|) = \min\{|A|,r\}$ for $r \in \mathbb{N}$. If $r=1$, then the solution of the problem above is always $f$ constant equal to $-1$ (that is, full collapse: nothing is learned). For $r=2$, there is a full collapse if $\pi_{\sigma(1)}(x) \leqslant 1/2$, while if $\pi_{\sigma(1)}(x) > 1/2$, $f_{\sigma(1)}$ is strictly larger than all other components (which are all equal), and the optimal prediction for 0-1 loss can be recovered as the unique largest value (the same consistency condition as structured support vector machine~\cite{liu2007fisher}). This is illustrated in \myfig{effect_concave}. More generally, for any $r$, the cost function $\varphi(|A|) = \min\{|A|,r\}$ will lead to a collapsed prediction only if $\pi_{\sigma(1)}(x) \leqslant 1/r$.

Note that it seems that considering a concave function $\varphi$ of $|A|$ can only be less efficient than using~$|A|$, as the optimal prediction function with $\varphi(|A|)$ is a non-injective function of the one for  $|A|$ (that is, some values are coalescing and some conditional probabilities cannot be recovered), and the trade-offs between the coverage of a set $A$ and $V(A)$ are equivalent if $V$ is replaced by an increasing function $\varphi \circ V$. However, this allows to learn ``simpler'' functions and can lead to more efficient estimation procedures than plainly using $|A|$ (see \mysec{experiments} for examples, as well as the simple situation where an estimation model enforces a small number of values of the prediction function, which the concave penalty clearly deals well with, while the non-concave penalty could cluster incorrectly). In other words, the loss function only focuses the modelling power of the prediction function towards outputs $y$ that have a chance to be included in the top predictions.

\paragraph{Beyond cardinality-based functions.} For discrete problems, various types of prior knowledge can be encoded to go beyond plain cardinality, such as the presence of groups (like in group Lasso) or hierarchies, in a similar spirit as for structured sparsity~\cite{bach2012structured}. One could also design specific submodular penalties for sets of permutations or problems with multiple labels.

 \begin{figure}
   \centering
     \includegraphics[width=.7\linewidth]{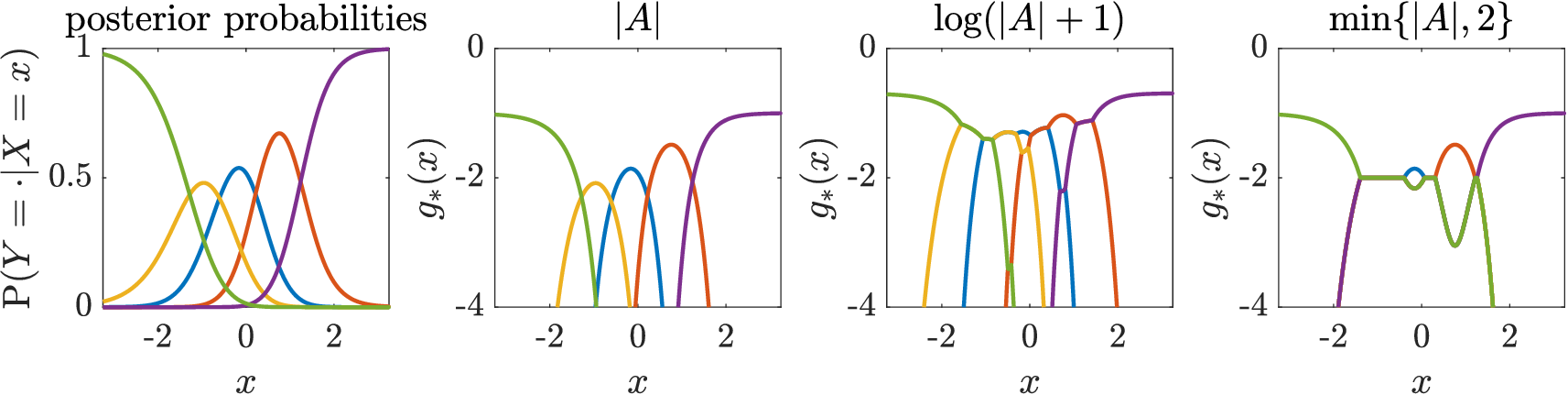}
     
     \vspace*{-.125cm}
     
     \caption{Effect of concave penalty. Left to right: posterior probabilities, optimal prediction functions for three concave penalties, for $\X=\rb$ and $k=5$ classes.  
     \label{fig:effect_concave}}
\end{figure}

\subsection{Regression}
\label{sec:regression}

In this section, we consider $\Y \in \rb^k$ equipped with a probability measure $M$ (typically uniform on a compact subset like in examples below, Gaussian, or with heavier tails such as a student distribution). We first consider the simple modular function and then more complex situations that favor certain types of sets.

\paragraph{Modular functions.}
With $\Y=\rb^k$ and $V(A)= \int_A dM (y)=M(A)$ with $M$ a probability measure, the cost function on $g:\X\times \Y \to \rb$ becomes
$$\ell(y,g(x)) = \int_\Y g(x,z) dM(z) + \frac{1}{2} g(x,y)^2,$$
 which has a natural unbiased estimate $g(x,z) + \frac{1}{2} g(x,y)^2$ where $z$ is sampled from $M$. This favors sets with small mass, regardless of their shape, with an optimal function equal to $ g^\ast(x,y) = - \big( \frac{dp(y|X=x)}{dM(y)}\big)^{-1}$.

\paragraph{Mathematical morphology based functions.}
We consider a set $K\subset \rb^k$ referred to as a ``structuring element,'' which we assume centrally symmetric (that is, $K=-K$), typically a ball of center $0$ and radius~$r$. We then consider the set-cover
\BEQ
\label{eq:morphomat}
V(A)= M\bigg(\bigcup_{y \in A} \big( \{ y \} +K \big) \bigg).
\EEQ
In the language of mathematical morphology~\cite{najman2013mathematical}, the set $\bigcup_{y \in A} \big( \{ y\} +K \big) $ is the \emph{dilation} of the set $A$. In contrast, the \emph{erosion} of a set $A$ is the set $\big\{y \in \Y,  \{y\} + K \subset A \big\}$. The closure is then the composition of the dilation and then the erosion of $A$, while the opening is the composition of the erosion and then the dilation of $A$. See illustrations in \myfig{morphomat}. Closed sets are sets equal to their closures, and open sets are sets equal to their openings. For our function $V$, sets $A$ have the same value as their closure, so possible estimated sets are all closed, that is, no small holes and no small isolated components.

\begin{figure}
\centering
     \includegraphics[width=.75\linewidth]{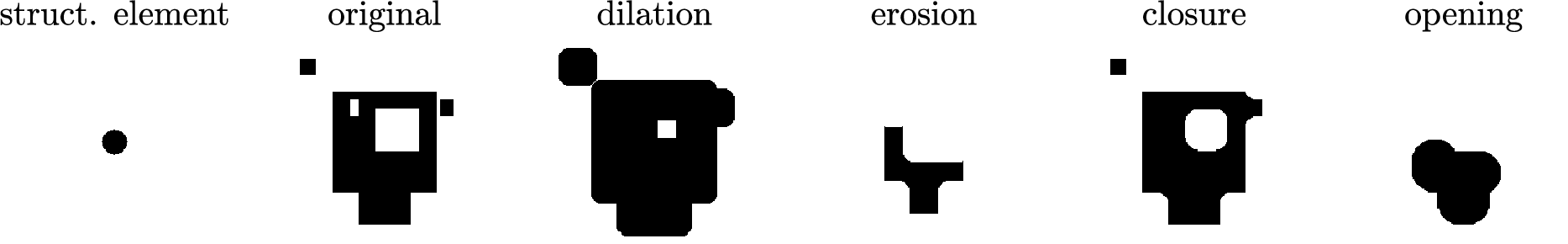}
     
     \vspace*{-.125cm}

\caption{Examples of erosion, dilation, opening, closure of a set from the structuring element (left panel) in two dimensions. The size function defined in \eq{morphomat} leads to sets that are equal to their closure, thus without small holes compared to the size of the structuring element.
\label{fig:morphomat}}\end{figure}

\paragraph{Choquet integral / \lova extension.}
It can then be shown that, for any function $f: \X \to \rb$,
$$v(f)=\int_\Y \Big\{\sup_{t \in K} f(z+t) \Big\}dM(z),$$
for which an unbiased estimate can be obtained as $\sup_{t \in K} f(z+t) $ where $z$ is sampled from $M$ (this is useful for optimization by stochastic gradient descent). The function $z \mapsto \sup_{t \in K} f(z+t)$  also defines a morphological operation (a dilation), now on real-valued functions and not only on binary-valued functions (which can be identified to sets)~\cite{najman2013mathematical}.   
In \mysec{algorithms}, we show how we can perform optimization using reweighted least-squares algorithms.

\paragraph{Optimal functions.} For a given $x \in \X$, the optimal function $f_\ast = g_\ast(x,\cdot)$ optimizes
$$ v(f) + \int_{\Y} f(y)^2 d\pi(y),$$
where $\pi(y) = p(y|X=x)$. For $\Y = \rb$, we show several optimal functions $f_\ast$ for a given probability distribution $\pi$ in \myfig{morphomat_estimates}: with increasing radius, there are more flat parts, and when taking level sets, the obtained sets would have fewer holes.

\begin{figure}
\centering
     \includegraphics[width=.8\linewidth]{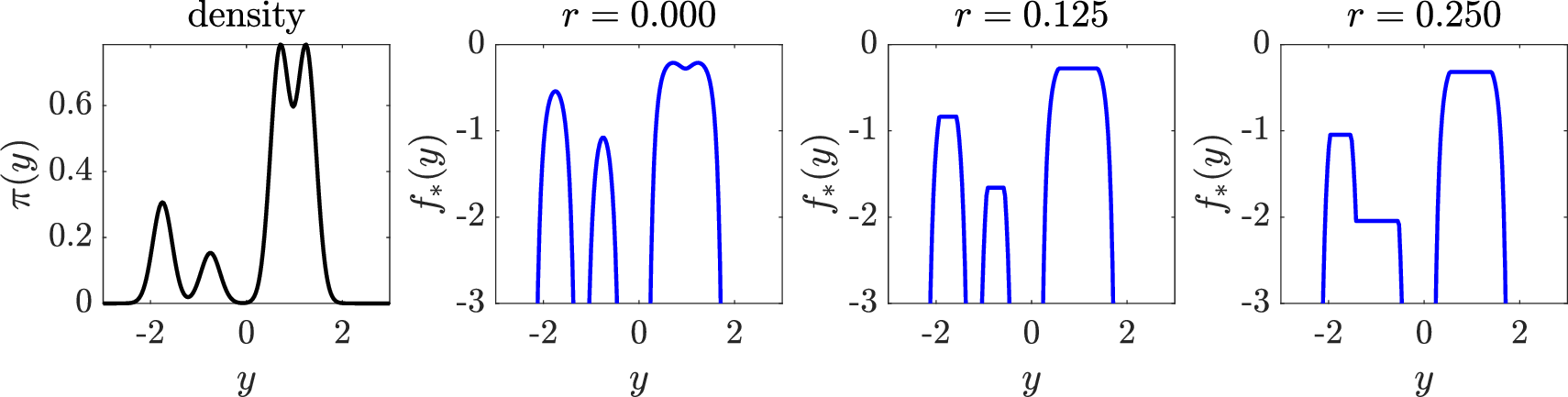}
     
     \vspace*{-.125cm}

\caption{Examples of estimation of functions. From left to right: density $\pi$, estimation with increasing radius $r$ of structuring elements (and increasingly larger piecewise constant parts). 
\label{fig:morphomat_estimates}}\end{figure}

 \subsection{Additional examples}
 \label{sec:addex}
 Like in structured sparsity~\cite{bach2012structured}, additional examples could be considered based on graphs (see~\cite{obozinski2016unified,bach2013learning}).
Some non-decreasing functions for which the problem in \eq{optAlag} is tractable are natural but are not submodular (similar developments could be carried out as future work, although the proper representation of sets and the associated layer-cake formulas remain to be determined), such as:
\BIT
\item Given an additive measure $M$ and a family $\mathcal{A}$ of subsets of $\Y$ (e.g., convex sets, ellipsoids, balls), $V(B) = \inf_{ A \in \mathcal{A}, \ A \supseteq B} M(A)$, that is, the smallest set (for $\mu$) in $\mathcal{A}$ containing $B$. We have $V(B) \geqslant M(B)$ and $V(A) = M(A)$ for $A \in \mathcal{A}$. This function is submodular if and only if the family $\mathcal{A}$ is a lattice (that is, closed under intersection and union); for example, the ancestor sets of a directed acyclic graph. Beyond submodularity, we would get tractable problems if $\mathcal{A}$ has a manageable size.

\item Given a family $\mathcal{A}$ of subsets of $\Y$, with a function $M: \mathcal{A} \to \rb_+$, the minimal weighted set cover
$$
V(C) = \inf_{  \mathcal{B} \subset \mathcal{A}, \ C \subseteq \bigcup_{B \in \mathcal{B} } \! B} \sum_{B \in \mathcal{B}} M(B),
$$
which has a traditional linear-programming relaxation that is tractable when the family $\mathcal{B}$ is sufficiently small.
Note that this function is typically not submodular~\cite{fujishige2005submodular}.

\EIT

\section{Alternative area-based loss function}
\label{sec:otherlosses}
Given  a real-valued function $g:\X \times \Y \to \rb$, which defines the subset-valued function $A: \rb \times \X \to \P(\Y)$ through
$$
A(\lambda,x) = \{ y \in \Y, \ g(x,y) \geqslant -\lambda \},
$$
the loss defined in \eq{loss} leads to ``Fisher-consistent'' estimation, that is, if $g$ minimizes $\E[ \ell(Y,g(X,\cdot))]$, then the subset-valued function $A$ leads for all $\lambda \in \rb$ to the optimal prediction for $\inf_{B \subseteq \Y} V(B) + \lambda \P( Y \notin B|X=x)$, and then to the optimal prediction for conditional coverage (see Prop.~\ref{prop:main}).

More generally, given a subset-valued function $A$ (which may or may not be obtained through level-sets of a real-valued function) and for which we will use the real variable $\nu$ to avoid confusion, other losses can be naturally defined based on the performance curve in the ``size vs.~coverage'' plane, akin to the receiver operating characteristic (ROC) curve and the area under it, that are commonly used in binary classification~\cite{hand2023notes,tharwat2021classification,berta2024classifier}.

\paragraph{Definitions through interpolations.}

For a given function $A: \rb \times \X \to \P(\Y)$ which is non-decreasing in its first argument, for a given $x \in \X$, we obtain a ``curve'' 
$$(s_A(\nu,x),\alpha_A(\nu,x))_{\nu \in \rb} = \big(V(A(\nu,x)),\P(Y \notin A(\nu,x)|X=x)\big)_{\nu \in \rb}$$ in the two-dimensional (size, coverage) plane, a curve which requires knowledge of the conditional distribution of $Y$ given $X=x$. Since the function $A$ is assumed non-decreasing in $\nu$, this curve is non-increasing in the plane $(s,\alpha)$, and belongs to $[0,V(\Y)] \times [0,1]$. However, in particular for discrete set $\Y$, for any $x \in \X$, the set $\{ s_A(\nu,x), \ \nu \in \rb\}$ is strictly included in $[0,1]$, and can be a finite set of points (i.e., when the function $g(x,\cdot)$ takes only $m(x)$ many values, then $m(x)+1$ sets $\varnothing \subseteq B_0(x) \subsetneq B_1(x) \subsetneq B_2(x) \subsetneq B_m(x) \subseteq \Y$ are possible). For simplicity, in this section, we assume that we are always in this discrete situation for all $x \in \X$ (extensions could be obtained by considering integrals instead of sums to define areas, with potentially sampling to estimate them).

Note that we can have two sets of the same size that are different and lead to different coverage probabilities (which can happen only when the function $V$ is not strictly increasing, e.g., with set covers).

In order to extend the curve into a continuous one, several approaches are possible, with four possible curves (and thus areas), as illustrated in \myfig{areas}:

\begin{figure}[h]
    \centering
    \begin{center}
\includegraphics[width=9cm]{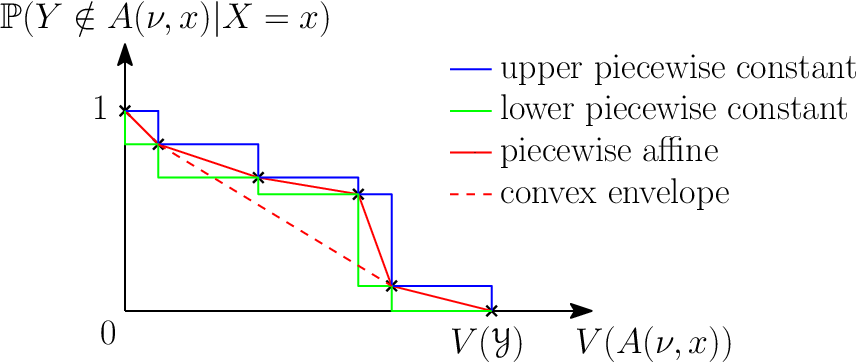}
\end{center}

\vspace*{-.3cm}

    \caption{Given a finite set of points $\big(V(A(\nu,x)),\P(Y \notin A(\nu,x)|X=x)\big)_{\nu \in \rb}$ (black crosses), three different interpolants can be defined, each with its own area. }
    \label{fig:areas}
\end{figure}

\BIT
\item \textbf{Upper piecewise constant interpolation} $(s_A^+(\nu,x),\alpha_A^+(\nu,x))_{\nu \in \rb}$. 
This defines ${\rm area}^+(A,x)$, the area below the curve.
When $A(\cdot,x)$ takes $m(x)$ distinct values as detailed above, this is equal to
$\sum_{j=1}^{m(x)} ( V(B_j(x))-V(B_{j-1}(x)) ) \P( Y \notin B_{j-1}(x) | X = x)$.

The final criterion, \emph{once averaged over $x$}, is $\E\big[{\rm area}^+(A,X) \big]$. Note that from a distribution given through i.i.d. samples $(x_1,y_1,\dots,x_n,y_n)$ (e.g., some test data), it can be estimated as 
\BEQ
\label{eq:estdata}
\frac{1}{n} \sum_{i=1}^n \sum_{j=1}^{m(x_i)} ( V(B_j(x_i))-V(B_{j-1}(x_i)) ) 1_{ y_i \notin B_{j-1}(x_i) }.
\EEQ
For values of $s$ that are not in the set $\{ s_A(\nu,x), \ \nu \in \rb\}$, this choice corresponds to outputting a single subset and is adapted to be conservative in terms of size (and is thus adapted when outputting a set of a given size or less).

\item \textbf{Lower piecewise constant interpolation}  $(s_A^-(\nu,x),\alpha_A^-(\nu,x))_{\nu \in \rb}$. This defines ${\rm area}^-(A,x)$. When $A(\cdot,x)$ takes $m$ distinct values, this is equal to
$$\sum_{j=1}^{m(x)} ( V(B_j(x))-V(B_{j-1}(x)) )     \P( Y \notin B_{j}(x) | X = x). $$ 
The final criterion is $\E\big[{\rm area}^-(A,X) \big]$, with an estimation formula given data similar to \eq{estdata}.

This choice corresponds to outputting a single subset and is adapted to be conservative in terms of coverage (and is thus adapted when outputting a set of a given miscoverage or less).

\item \textbf{Piecewise affine interpolation} $(\bar{s}_A(\nu,x),\bar{\alpha}_A(\nu,x))_{\nu \in \rb}$: if randomized predictions are allowed, then the piecewise affine interpolant of the curve can be achieved (in the example above, by predicting $B_j(x)$ with probability $q$ and $B_{j+1}(x)$ with probability $1-q$ for any $q \in [0,1]$).  This defines ${\rm area}(A,x)
= \frac{1}{2}{\rm area}^+(A,x) + \frac{1}{2}{\rm area}^-(A,x)$ for a given $x \in \X$.
This is equal to
$$\sum_{j=1}^{m(x)} ( V(B_j(x))-V(B_{j-1}(x)) ) \cdot \frac{1}{2}( \P( Y \notin B_{j-1}(x) | X = x) + \P( Y \notin B_{j}(x) | X = x) )$$ when $g$ takes $m(x)$ distinct values.
The final criterion is $\E\big[ {\rm area}(A,X) \big]$, with an estimation formula given data similar to \eq{estdata}.

\item \textbf{Convex envelope} 
$(\bar{s}^{\rm convex}_A(\nu,x),\bar{\alpha}^{\rm convex}_A(\nu,x))_{\nu \in \rb}$, that is, the largest convex function which is below all the curve: it can also be achieved by randomized predictions, and can be computed using isotonic regression~\cite{orderrestricted}. This leads to the criterion $\E\big[ {\rm area}^{\rm convex}(A,X) \big]$, which cannot be estimated from finite data as it requires having access to conditional probabilities to compute the convex hull. 
\EIT

The following proposition shows that two of these criteria based on areas are minimized at the same point as the one based on scalarization of the multi-objective optimization of both volume and coverage (which the Lagrangian formulation aims to do).

\begin{proposition}
\label{prop:optseq}
    The criteria $\E\big[ {\rm area}(A,X) \big]$ and $\E\big[ {\rm area}^{\rm convex}(A,X) \big]$ are minimized at the function $A$ defined as $A(\lambda,x) = \{ g(x,\cdot) \geqslant -\lambda\}$, where $g$ is the minimizer of the loss function $\ell$ defined in \eq{loss}. 
\end{proposition}
\begin{proof}
For a submodular criterion $V$, for all $x \in \X$, from Prop.~\ref{prop:main}, the optimal randomized predictions are obtained from the optimal function $g$, and the associated curve is convex, hence the result.
\end{proof}
The criterion $\E\big[ {\rm area}(A,X) \big]$ can be estimated from data and thus serves as a natural loss function for non-decreasing functions, with an intuitive alternative formulation which we describe next. Note that when applied to set functions obtained from level sets, it is not convex (and in fact not continuous if level sets are considered), and thus cannot be easily used in estimation procedures. 

\paragraph{Alternative formulation.} Through discrete integration by parts for a finite sequence of sets, we have:
\BEAS
{\rm area}^-(A,x)
& = &  \sum_{j=1}^{m(x)} ( V(B_j(x))-V(B_{j-1}(x)) ) \P( Y \notin B_{j-1}(x) | X = x) \\
& = &  \sum_{j=1}^{m(x)}  V(B_j(x))  \P( Y \in B_{j}(x) \backslash B_{j-1}(x) | X = x)
,\EEAS
which is equal to the expected size of the smallest set containing $Y$, which is a common criterion in information retrieval~\cite{manning2008introduction}, here used on a single observation.

In the general case, we can define
\BEAS
v_A^+(x,y) & = &  \sup \big\{ V(A(\nu,x)), \ \nu \in \rb, \ y \notin A(\nu,x) \big\} \\
v_A^-(x,y) & = &  \inf \big\{ V(A(\nu,x)), \ \nu \in \rb, \ y \in A(\nu,x) \big\}, \\
v_A(x,y) & = & \frac{1}{2} v_A^+(x,y)  + \frac{1}{2} v_A^-(x,y) ,
\EEAS
where $v_A(x,y)$ corresponds to the performance of randomized prediction. We then have $\E\big[ {\rm area}^{\pm}(A,X) \big] 
= \E\big[ v_A^{\pm}(X,Y) \big]$, and 
$\E\big[ {\rm area}(A,X) \big] 
= \E\big[ v_A(X,Y) \big]$.

In our experiments in \mysec{experiments}, we compute these ``area-losses'' to compare several estimators.

\paragraph{Averaged curves.}
Given the curves in \myfig{areas} for each $x \in \X$ (which require the knowledge of conditional probabilities to be drawn), an aggregated curve can be obtained so that the area under the curve corresponds to $\E\big[ {\rm area}(A,X) \big]$ or $\E\big[ {\rm area}^\pm(A,X) \big] $ (these can be computed \emph{without} access to the conditional distribution). In order to aggregate several curves corresponding to several potential $x$, a common parameter has to be chosen; beyond $\nu$, the only one that can be computed is the size $V(A(\nu,x))$, and thus, for each $s \in [0,V(\Y)]$, we can define $\nu(s,x)$ such that $V(A(\nu(s,x),x))=s$ (this is only possible for randomized predictions where all values of $s$ can be achieved in expectation) and it is possible to average all corresponding $\alpha$-values over all $x$, to obtain a curve where for each point $(s,\alpha)$ on the curve, we have a fixed size for every $s$ (and every $x$), and on average (over $x$) the miscoverage $\alpha$ (that is, a marginal coverage). 

 Note that these curves can be used to assess coverage for a given~$\alpha$ only for marginal coverage, but the area of the affine interpolation is exactly $\E\big[ {\rm area}(A,X) \big]$, and thus can be used to evaluate conditional coverage. Its convex envelope can be computed, which is also an aggregated measure of conditional coverage, and can be used to assess marginal coverage with randomized predictions. 
 
In \mysec{conformal}, we will show how, for our loss function, we can estimate a parameterization by the miscoverage level $\alpha \in (0,1)$.

\section{Obtaining high-probability sets and conformalization}
\label{sec:conformal}
Given a candidate function $g:\X \times \Y \to \rb$ that approximately minimizes the loss function defined in \eq{loss}, we can generate several outputs that can be used in different setups. For all of them, we can choose either deterministic predictions or randomized predictions (with two possible sets). All will correspond to a certain curve in the $(s,\alpha)$-plane for each given $x$, and a corresponding averaged curve with different semantics. We already described in \mysec{otherlosses} how the trade-off parameter $\lambda$ or the size $s$ could be used for aggregation. We now show how the miscoverage $\alpha$ can be used, which can only be done approximately through an explicit estimation of the conditional probability.

\paragraph{Estimate of conditional probability.}
This section is based on the fact that after learning, $g(x,\cdot)$ is an approximate minimizer of $v(g(x,:)) + \frac{1}{2}\E[ g(x,Y)^2|X=x ]$. If all values of $g(x,\cdot)$ are distinct (which is to be expected if enough randomness is present in our estimation procedure), a unique subgradient of~$v$ is a positive additive measure $\mu(\cdot|x)$ such that $\mu(\Y|x) = V(\Y)$ and $\mu(A|x) \leqslant V(A)$ for all $A \subseteq \Y$. The optimality condition is then
$$
\forall y \in \Y, \ 
d\mu(y|x) + g(x,y) dp(y|X=x) = 0,
$$
leading to an estimate $d \hat{p}(y|X=x) = - d\mu(y|x) / g(x,y)_-$, which is a non-negative measure which we normalize to one (by dividing by its sum). 
Given this estimate $\hat{p}$, we can solve exactly the following problem to get all (estimated) prediction sets at all coverage levels
\BEQ
\label{eq:AF}
\inf_{f: \Y \to \rb} v(f) + \frac{1}{2} \int_\Y f(y)^2 d \hat{p}(y|X=x) .
\EEQ
Indeed, denoting $\hat{f}(x,\cdot)$ the minimizer of \eq{AF}, owing to Lemma~\ref{lemma:calibration}, the sets $\{ \hat{f}(x,\cdot) \geqslant -\lambda\}$ are optimal for the minimization of $V(B) + \lambda \hat{\P}(B^{\sf c}|X=x)$.
Note that it turns out that a constant times $g(x,\cdot)_-$ is such a minimizer, so no extra problem has to be solved and no extra clustering of values is to be expected.

\paragraph{Thresholds for fixed conditional coverage.}

If a fixed conditional coverage level $\alpha$ is desired, we now provide an algorithm to compute $\lambda(\alpha,x)$ and the associated set $\{ \hat{f}(x,\cdot) \geqslant -\lambda(\alpha,x)\}$ for deterministic prediction, as well as $\lambda_\pm(\alpha,x)$ for randomized predictions, such that the set
$\{ \hat{f}(x,\cdot)  \geqslant -\lambda(\alpha,x)\}$ (and its randomized counterpart) 
provides a good estimate of the minimizer of $V(B)$ such that $\P( Y \notin B | X= x) \leqslant \alpha$ (note that this minimization problem only has a fractional solution that can be obtained by randomization). Our procedure is exact with the optimal prediction function that minimizes the loss in \eq{loss}, and only approximate otherwise.

The thresholds are obtained directly from \myfig{areas}, by looking at a fixed horizontal level $\alpha$, which has to lie in an interval between the coverage levels of sets corresponding to two values of $\lambda$, which we denote $\lambda_+(\alpha,x)$ and $\lambda_-(\alpha,x)$. For deterministic predictions, we simply select $\lambda(\alpha,x)$ as the largest of the two (which leads to the larger set). For randomized predictions, the probability associated with each of the two values of $\lambda$ is selected in $[0,1]$ to precisely match the expected coverage $\alpha$.

Note that since $\hat{g}$ is obtained by minimizing $\E[ v(g(X,\cdot) + \frac{1}{2} g(X,Y)^2]$ and the optimal $g^\ast$ is non-positive, by strong convexity, an excess risk of $\varepsilon$ is the optimization of $\hat{g}$ leads to a bound
$\frac{1}{2} \E[ (\hat{g}(X,Y)_- -  g^\ast(X,Y) )^2 ] \leqslant \varepsilon$, which could in turn lead to approximation guarantees if extra smoothness assumptions are made on $\hat{g}$ and~$g^\ast$.

 \paragraph{Conformalization.}

Given a candidate $\lambda^\ast(\alpha,x)$, we could ``conformalize'' it using standard conformal prediction with the scores $\hat{f}(x,y) + \lambda^\ast(\alpha,x)$ or $\hat{f}(x,y) / \lambda^\ast(\alpha,x)$ using split conformal prediction~\cite{angelopoulos2021gentle}, to at least get certified marginal coverage. This would lead to the estimation of a threshold $\hat{q}_\alpha$ such that $\P( \hat{f}(X,Y) + \lambda^\ast(\alpha,X) \geqslant \hat{q}_\alpha)$ or 
$\P( \hat{f}(X,Y)  \geqslant \hat{q}_\alpha \lambda^\ast(\alpha,x) )$ are exactly in an interval $[1-\alpha,1-\alpha +1/(n+1)$, where $n$ is the size of the calibration set.

Our setup also applies to any form of density estimation: given any estimate of $p(y|X=x)$, we can minimize for any $x$, $v(f) + \frac{1}{2} \int_\Y f(y)^2 d\hat{p}(y|X=x)$ and then obtain the desired sets as done above.

\section{Optimization algorithms}
\label{sec:algorithms}
\label{sec:guarantees}

In this section, we first describe two sets of optimization algorithms, one dedicated to large-scale potentially non-linearly parameterizable predictors (such as neural networks) based on stochastic gradient descent, and one dedicated to linearly parameterizable predictors (such as using positive definite kernel methods) based on an iterative reweighted least-squares formulation~\cite{daubechies2010iteratively,bach2012optimization}, so that we can obtain precise solutions. We only consider the latter in our experiments in \mysec{experiments}. Theoretical guarantees based on standard learning theory and optimization guarantees~\cite{bach2024learning} could also be obtained.

\paragraph{Additional regularization.} 
In both cases, we add two modifications for better empirical behavior:
\BIT
\item \textbf{Label smoothing}: We can add the penalty $\varepsilon \int_\Y g(x,z)^2 dM(z)$ for some positive measure $M$ to make the loss strongly convex in $g$, which stabilizes the estimation (we select $\varepsilon=10^{-2}$ throughout experiments). All of our examples have a natural such measure such that $V \geqslant M$ and $V(\Y) = M(\Y)$. When this is applied, then the conformalization procedures from \mysec{conformal} have to be adapted since now the optimal $g(x,\cdot)$ depends on $\P(\cdot|X=x) + \varepsilon M$ (it is thus a form of label smoothing), and this implies a new estimate of the conditional probability (that has to be projected on the set of probability measures).

\item \textbf{Post-clustering}: Because of the submodular penalty, optimal prediction functions typically have clustered values, which usual parameterized models cannot exactly enforce. Such clustered behaviors can be favored by decomposing $V$ as $V-W+W$, where both $V-W$ and $W$ are submodular, and $W$ is non-decreasing. We can then minimize at training time the loss $\E[ w(h(X,\cdot)) + \frac{1}{2}  h(X,Y)^2]$, and at testing time, given $\hat{h}(x,\cdot)$, to obtain the final estimate $g(x,\cdot)$, minimize $v(f) + \frac{1}{2} \int_\Y \frac{f(y)^2}{-\hat{h}(x,y)_-} d\mu(y|x)$, which is an extension of total variation regression (see~\cite{bach2011shaping} and references therein), that leads to clustered values and thus to enhanced estimation for randomized predictions.

A natural possibility for $W$ is $W=M$ modular measure that $V$ dominates, which leads to a particularly simple training procedure (quadratic loss) and keeps all non-linear clustering behaviors for testing. This, however, leads to learning $h$ instead of $g$, which is more than desired (that is, it cannot leverage the fact that $g$ has clustered values, while $h$ does not). To mitigate this effect, we can take $W = \varepsilon M + ( 1- \varepsilon) V$, for $\varepsilon$ small (this will cluster values of $h$ that are close enough).

\EIT

\paragraph{Oracles for size function $V$.}  In our examples for different sets $\Y$, we will need the following oracles for the function $V$ and its Choquet integral / \lova extension $v$:
\BIT
\item Subgradient of $v(f)$: given $f$, the subgradient of $v$ is $V(\Y)$ times a probability measure (an element of the core of $V$, as defined in \mysec{lova}). To minimize our loss function using stochastic gradient descent, we only need a sample from this measure, whereas we need the full measure to obtain sets with fixed conditional coverage in \mysec{conformal}. This is possible for all our examples in \mysec{examples}.

\item Minimization of $V$ plus a modular function (needed to obtain optimal solutions given conditional probabilities): A classical result from submodular analysis shows that this is equivalent to being able to minimize $v$ plus a convex separable function (see \mysec{lova} and \cite[Chapter 9]{bach2013learning}). This oracle is needed for the two-step procedure described just above that leads to additional clustered values, and for the procedure to obtain conditional coverage outlined in \mysec{conformal}.

\item Reweighted quadratic formulation: $v$ is rewritten as $v(f) = \inf_{ \eta \geqslant 0} \frac{1}{2} \int_\Y \frac{f(y)^2}{\eta(y)} dM(y) + \frac{1}{2} \gamma(\eta)$, and can be smoothed using a minimization with respect to $\eta \geqslant \varepsilon$. This is what we focus on in our simulations in \mysec{experiments}.

\EIT

\subsection{Stochastic gradient descent (SGD)}
The loss function is convex and subdifferentiable, and hence we can apply any classical optimization algorithm, such as SGD, with the usual guarantees. In our context, we need to learn a function $g: \X \times \Y \to \rb$, which we can parameterize arbitrarily. The function $v(g(x,\cdot))$ is, however, hard to compute, but in all our examples, an unbiased estimate of a subgradient can be obtained easily. This allows running optimization algorithms for any prediction models, not necessarily linearly-parameterized (such as neural networks)

\subsection{Iteratively-reweighted least-squares algorithms}
Given that one part of the loss function is quadratic (the part $\frac{1}{2} g(x,y)^2$), we can use only quadratic optimization (by solving linear systems), if we can treat the non-differentiable part appropriately, using reweighted-least-squares formulations~\cite{daubechies2010iteratively,bach2012optimization}.

We consider kernel methods as predictors to focus on the differences in loss functions without the need to worry about optimization issues (as this leads to prediction functions that are linear in their parameters).

\paragraph{Predictors.}
Kernel methods are used with incomplete Cholesky decomposition~\cite{fine2001efficient,bach2005predictive}. We consider a positive definite kernel such as $k(x,y) = \exp(-\alpha\|x-y\|_2)$ or $k(x,y) = ( 1 + \alpha x^\top y)^r  $, or the conditional positive kernel~\cite{wendland2004scattered} $k(x,y) = -\| x-y\|_2$ (where we assume that we know a constant $s$ such that $K + s 1_n1_n^\top$ is positive definite, this can be any $s>-1/1_n^\top K^\dagger 1_n$, or, according to~\cite{bach2023relationship}, $s= \frac{2 R}{\sqrt{\pi}} \sqrt{d/2} \geqslant \frac{2 R}{\sqrt{\pi}} \frac{\Gamma((d+1)/2)}{\Gamma(d/2)} $, where $R$ is the radius of an $\ell_2$-ball containing the data). For a positive definite kernel, we take $s=0$.

We then create an empirical feature map by selecting greedily (see~\cite{fine2001efficient,bach2005predictive}) a set $I \subset \{1,\dots,n\}$, computing a Cholesky decomposition of $K_{II} = GG^\top$ (the submatrix of $K$ with columns and rows index by $I$), and considering the feature map $\varphi(x) = G^{-\top} (k(x,x_i))_{i \in I} $, this leads, on the training data, to an approximation of the kernel matrix $\hat{K} = K_{\cdot I} (GG^\top)^{-1} K_{\cdot I} $, for which we have $\hat{K}_{II} = K_{II}$ (in our experiments, we chose to select $I$ such that $\| K - \hat{K} \|_\ast \leqslant \varepsilon \| K \|_\ast = \varepsilon \tr(K)$, with $\varepsilon = 10^{-3}$, where $\| \cdot \|_\ast$ is the nuclear norm). We then consider a predictor of the form $f(x) = \theta^\top \varphi(x) + \beta $, with penalty $\| \theta\|_2^2$.

\paragraph{Discrete outputs.}

We consider $\Y = \{1,\dots,k\}$, and for  $V(A) = \frac{1}{k}|A|$, we parameterize $g: \X \to \rb^k$. Given $n$ observations $(x_i,y_i) \in \X \times \Y$, $i=1,\dots,n$, and $A_i = \{ i \in \{1,\dots,n\} ,\ y_i = j\}$, this leads to an objective function where each label can be treated independently:
$$
\sum_{j=1}^k  \Big\{
\frac{1}{kn} \sum_{i=1}^n g_j(x_i)
+ \frac{1}{2n} \sum_{i \in A_j} g_j(x_i)^2  + \frac{\lambda}{2} \| g_j\|^2 \Big\},$$
where $\| g_j\|^2$ is the penalty described above (that is, $\| g_j\|^2 = \| \theta_j\|_2^2$, if $g_j=\theta_j^\top \varphi(\cdot) + \beta_j$).
As mentioned earlier, we add an extra penalty $\frac{\varepsilon}{k} \frac{1}{2n} \sum_{i =1}^n g_j(x_i)^2 $ for stability, with $\varepsilon = 10^{-2}$. If $d$ is a dimension of the feature space (i.e., $d$ is the number of used columns in the approximation of the kernel matrix), then the overall complexity is $O(k d^2n)$ per iteration, when solving the linear systems by Gaussian elimination, and $O(k dn)$ if using conjugate gradient.

For more generic functions of cardinality, in order to perform optimization, we can  use the following representation, for the sum of $r$ largest elements of $g(x)$, with $g_{\sigma(1)}(x) \geqslant \cdots \geqslant g_{\sigma(k)}(x)$:
\BEAS
\sum_{i=1}^r g_{\sigma(i)}(x)
& = & \min_{ t_r \in \rb} r t_r + \sum_{i=1}^k ( g_i(x) - t_r)_+= \min_{ t_r \in \rb} r t_r + \frac{1}{2}\sum_{i=1}^k ( g_i(x)  - t_r) + \frac{1}{2}\sum_{i=1}^k |g_i(x) -t_r|
\\
& = & \min_{ t_r \in \rb, \eta_r \in \rb_+^k} \frac{1}{2} g(x)^\top 1_k + ( r - \frac{k}{2} ) t_r + \frac{1}{2} \sum_{i=1}^k \frac{(g_i(x) -t_r)^2}{2\eta_{ri}}  
 + \frac{1}{2} \sum_{i=1}^k \frac{\eta_{ri}}{2},
\EEAS
with an optimal $t_r$ between $g_{\sigma(r+1)}(x) $ and $g_{\sigma(r)}(x)$, and $\eta_{ri} = |g_i(x)-t_r|$. The function above is jointly convex in $g(x),\eta_r, t_r$, and thus we can use it within an alternate ``reweighted quadratic'' optimization framework, by alternating between finding $\eta, t$ (in closed form) and optimizing with respect to $g$ (with a linear system). For this, following~\cite[Section 5]{bach2012optimization}, it is preferable to avoid that $\eta_{ri}$ is too small. This can be obtained by adding a constraint that $\eta_{ri} \geqslant \varepsilon$ (all $\eta_{ri}$ can then be obtained in closed form and the $t_r$ using binary search). The overall complexity is then a constant times the one for modular penalties.

\paragraph{Regression.}

For simplicity of implementation, we will consider a partition $C_1\cup \cdots \cup C_k = \Y$, and discrete predictions in one of these cells leading to an output set $\bar{\Y} = \{1,\dots,k\}$, with the underlying assumption that the cells are small. This allows to learn $k$ different functions $g_1,\dots,g_k$ (like in classification), but with an extra Laplacian penalty
$$
\sum_{i,j=1}^k w_{ij} ( g_i(x) - g_j(x) )^2
$$
to enforce smoothness across cells (we simply use weights that approximate when $\Y \subset \rb$, the squared $L_2$-norm of the derivatives in $y$). Then, the Choquet integral is equal to
$$v(\bar{f})
= \sum_{i=1}^k w_i \max_{ j \in C_j} \bar{f}_j,
$$
which is a sum of max functions where $C_j \subset \bar{\Y}$ and $w_j \in \rb_+$. This can be solved using iterated least-squares using the reweighted least-squares formulation above (corresponding to $r=1$):
\BEAS
v(\bar{f})
& = &  \sum_{i=1}^k w_i \max_{ j \in C_j} \bar{f}_j = \inf_{t_i \in \rb} \sum_{i=1}^k w_i \Big[ t_i + \sum_{ j \in C_j} (\bar{f}_j  - t_i)_+ \Big]\\& = & \inf_{t_i \in \rb} \sum_{i=1}^k w_i \Big[ t_i + \frac{1}{2}\sum_{ j \in C_j} (\bar{f}_j  - t_i) 
+ \frac{1}{4}\sum_{ j \in C_j} \frac{(\bar{f}_j  - t_i) ^2}{\eta_{ij}}+ \frac{1}{4}\sum_{ j \in C_j}  \eta_{ij} \Big].
\EEAS
We can then use preconditioned conjugate gradient algorithms~\cite{golub2013matrix} to solve linear systems, for which we can obtain a complexity in $O(k d n+ k^2 n)$ per iteration.

\section{Experiments}
\label{sec:experiments}

We now provide illustrations of our new loss functions.\footnote{Matlab code to reproduce all experiments can be downloaded from \url{www.di.ens.fr/~fbach/
submodular_conformal.zip}.}

\subsection{Discrete ouputs}
Here, we compare three different loss functions with kernel methods: the classical square loss (by solving the associated linear system), the multinomial loss from softmax regression (using the SAGA~\cite{defazio2014saga} algorithm to solve the optimization problem), and the new quadratic loss function (by solving the associated linear system or through an iterative least-squares algorithm).
Our aim is to understand the following phenomena:
\BIT
\item \textbf{Choice of non-parametric kernel:} In \myfig{kernelchoice}, we consider three types of non-parametric kernels, the exponential kernel without an unregularized constant term, the exponential kernel with an unregularized constant term, and the spline kernel, which behave increasingly better at extrapolation. From now on, we only consider the spline kernel when dealing with non-parametric estimation.

\begin{figure}[h]
   \centering
     \includegraphics[width=0.7\linewidth]{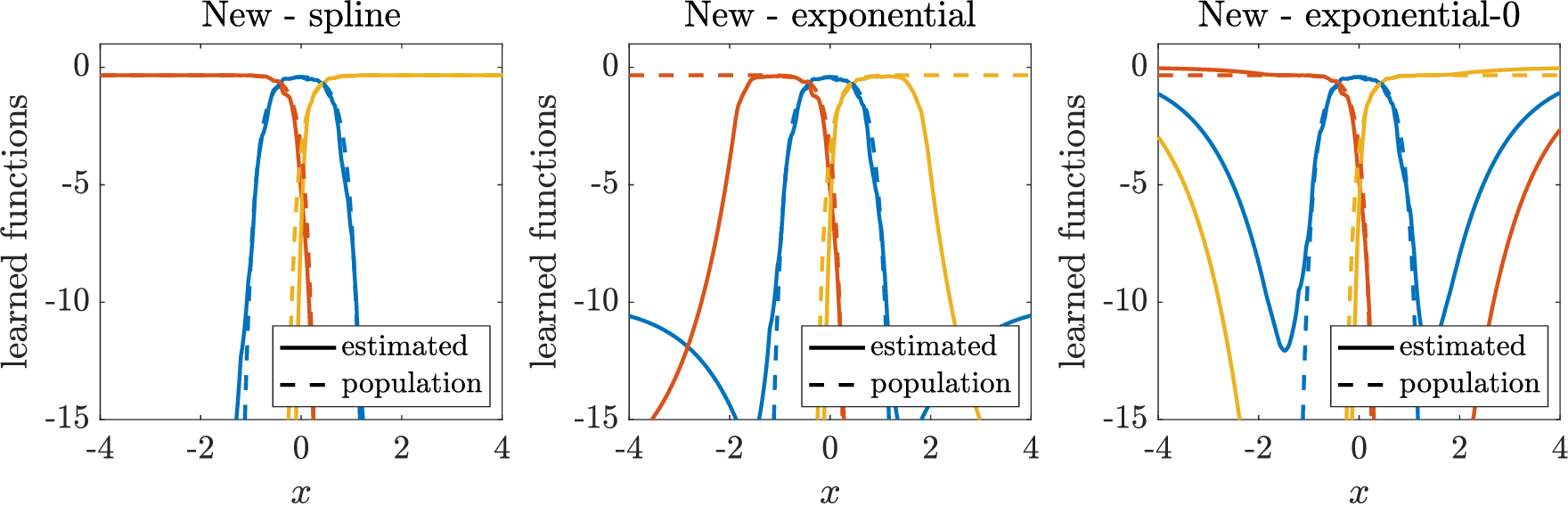}
     
     \vspace*{-.25cm}

     \caption{Comparing kernels on Gaussian class-conditional data in one dimension for classification with $k=3$ classes on a one-dimensional problem, with the new quadratic loss function corresponding to $V(A) = |A|/k$. From left to right: spline, exponential with an unregularized constant term, exponential without an unregularized constant term. \label{fig:kernelchoice}
     }
\end{figure}

\item \textbf{Impact of kernel choice on learned functions:} In \myfig{kernelimpact}, we see that (1) linear and quadratic kernels lead to underfitting, (2) with linear kernels, the classical quadratic loss is subject to the ``masking problem'' where a class is masked by others~\cite[Section 2.4]{hastie2009elements} while the new one is not, a problem that does not occur anymore for the quadratic kernel, (3) the softmax loss is best because here is well-specified (that is, the prediction function is linear for the softmax loss). This is not the case when using mixtures of Gaussians in later experiments.

\begin{figure}[h]
   \centering
   
   \vspace*{.25cm}
   
     \includegraphics[width=0.7\linewidth]{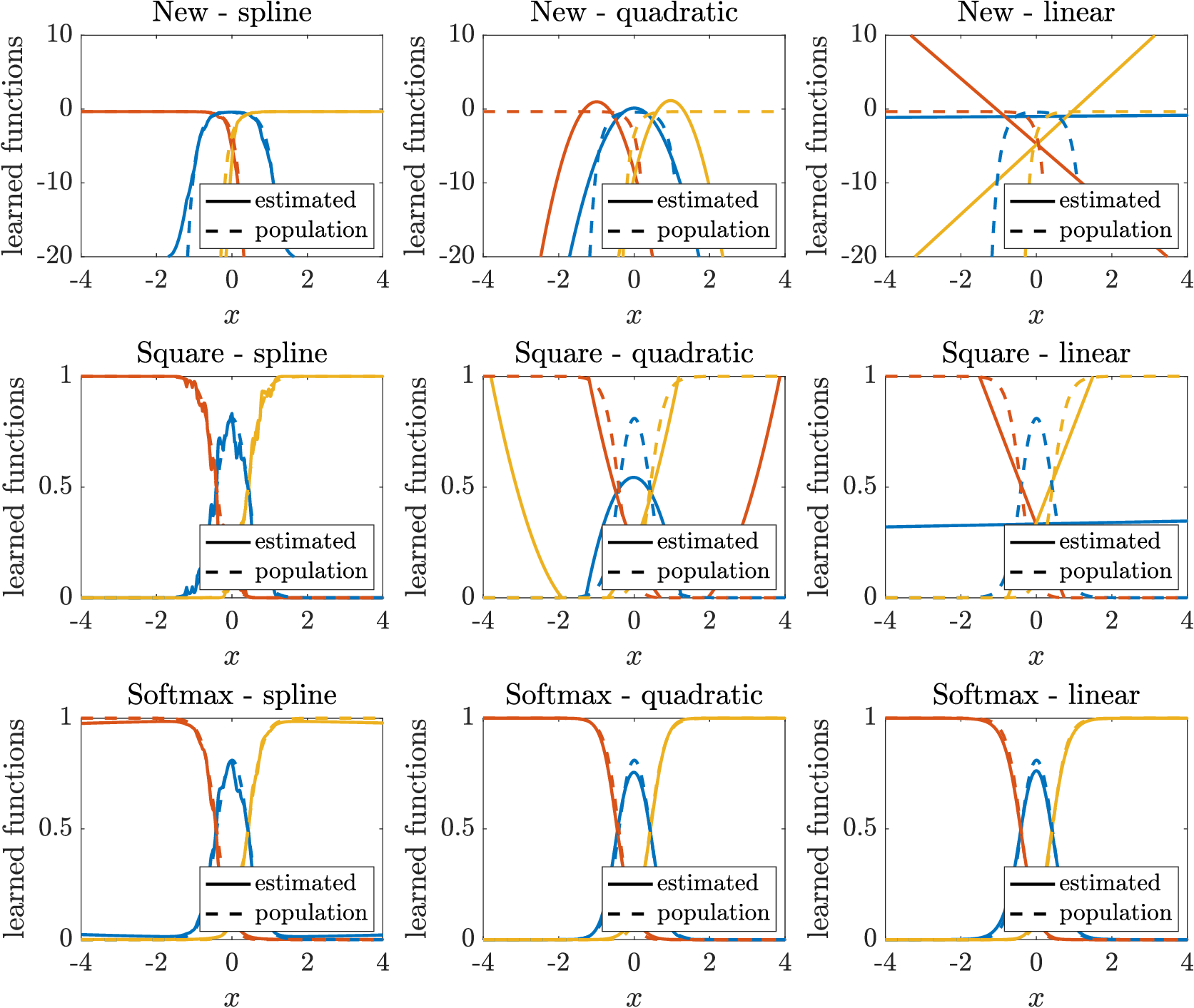}

   \vspace*{-.25cm}
   
   \caption{Comparing loss functions and kernels on Gaussian class-conditional data in one dimension for $k=3$ classes. Top: new loss function for $V(A) = |A|/k$, middle: regular quadratic function, bottom: softmax loss. From left to right: spline, quadratic, linear kernel.  
        \label{fig:kernelimpact}}
\end{figure}

\item \textbf{Comparisons of area losses:} In \myfig{ranksdis}, we compare the area loss function (for $V(A)=|A|$) on a four-dimensional problem with mixture of Gaussians class conditional data. We see that the softmax loss is better than the new square loss, which is better than the regular square loss for misspecified problems (low-rank kernels), and not for non-parametric modeling. We also compared the two versions of our method: regularized (with a quadratic penalty on all values of $g$, which corresponds to label smoothing, as described in \mysec{algorithms}) and unregularized, showing the benefits of label smoothing.

 In \myfig{ranksdisconcave}, we consider our new loss functions with the concave penalty $V(A) = \log( 1 + |A|)$, showing the benefits of using the concave penalty explicitly when learning, in particular for underparameterized models.  Moreover, we see how randomized predictions lead to smaller loss values.

\begin{figure}[h]
   \centering
     \includegraphics[width=0.7\linewidth]{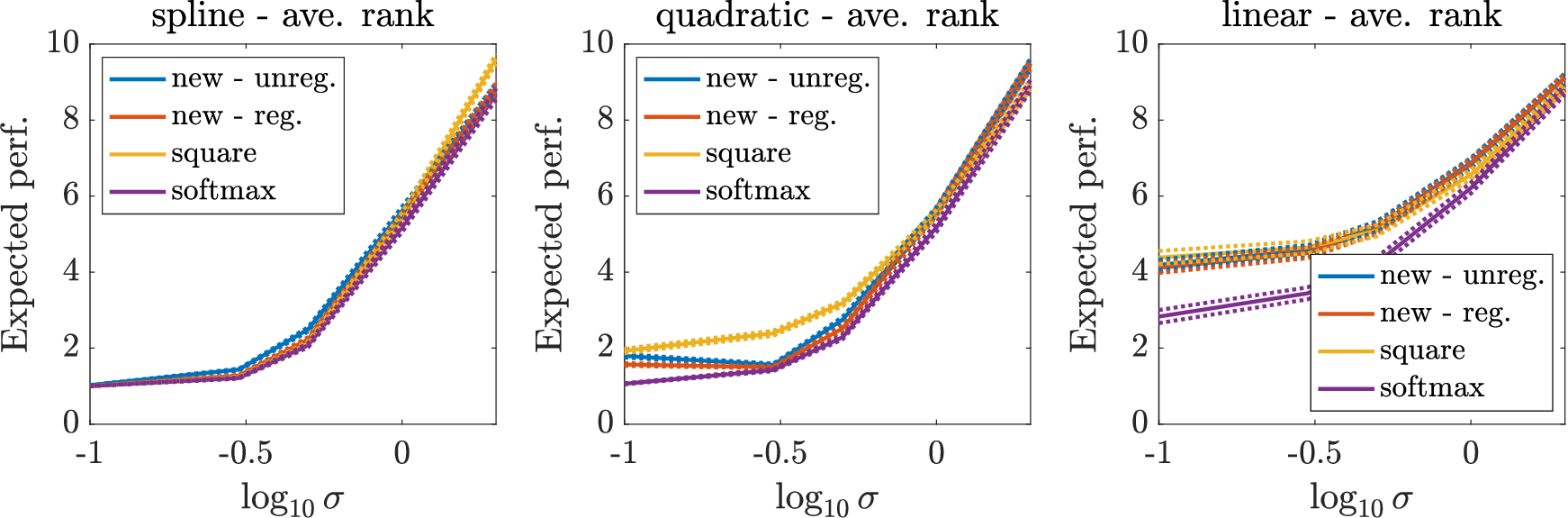}
     
     \vspace*{-.25cm}

     \caption{Comparing loss functions for $k=24$ in dimension $4$ and mixture of two Gaussians of variance $\sigma^2$ conditional data, averaged over ten replications (with error bars). When $\sigma$ is small, the prediction problem is easier (hence smaller losses), while it is harder for larger $\sigma$. We plot area-loss performances as a function of $\sigma$.
     Our new cost function is considered with $V(A)=|A|$, and with (``reg.'')  or without  (``unreg.'')extra label smoothing regularization. Left: spline kernel, middle: quadratic kernel, right:  linear kernel.   
     \label{fig:ranksdis}}
\end{figure}
 
 \begin{figure}[h]
   \centering
     \includegraphics[width=0.7\linewidth]{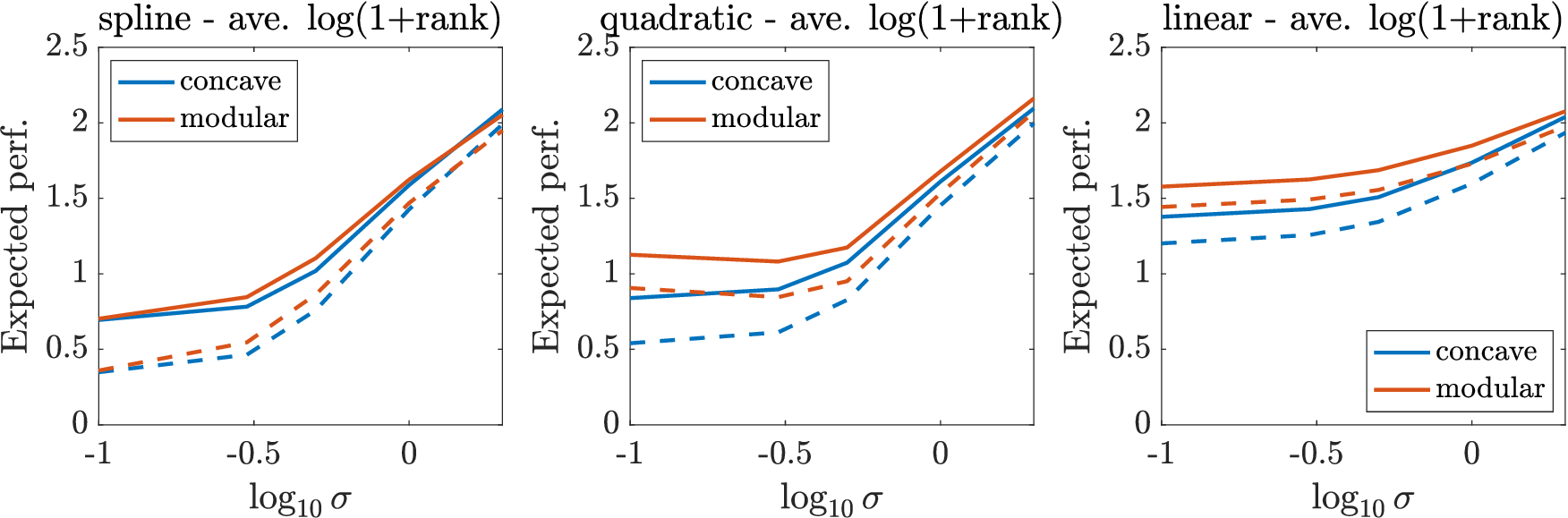}
   
     \vspace*{-.25cm}
     
      \caption{Comparing area loss functions for $k=24$ and mixture of two Gaussians conditional data in the same setting as \myfig{ranksdis}; comparing concave penalties and regular penalties. Left: spline  kernel, middle: quadratic  kernel, right:  linear kernel.  In each plot, we compare the new loss function, which is used with the modular function $V(A) = |A| \log(1+k) /k$, which is close to the one plotted in \myfig{ranksdis}, while we consider the concave function $V(A) = \log(1+|A|)$.   In plain performance for the deterministic prediction is shown, while randomized predictions are shown in dashed. 
     \label{fig:ranksdisconcave}}
\end{figure}

\item \textbf{Comparisons of conditional coverage guarantees:}
For a learned function $g$ for a cardinality problem with the spline kernel, for a ``fixed $\alpha$'' curve, we can check the conditional coverage property for each $x$ in \myfig{coverage}, where we see that methods relying on optimizing marginal coverage do not lead to good conditional coverage. Moreover, the optimal prediction function leads to perfect conditional coverage, while post-clustering slightly improves the performance of the learned prediction function.

\begin{figure}[h]
   \centering
     \includegraphics[width=0.95\linewidth]{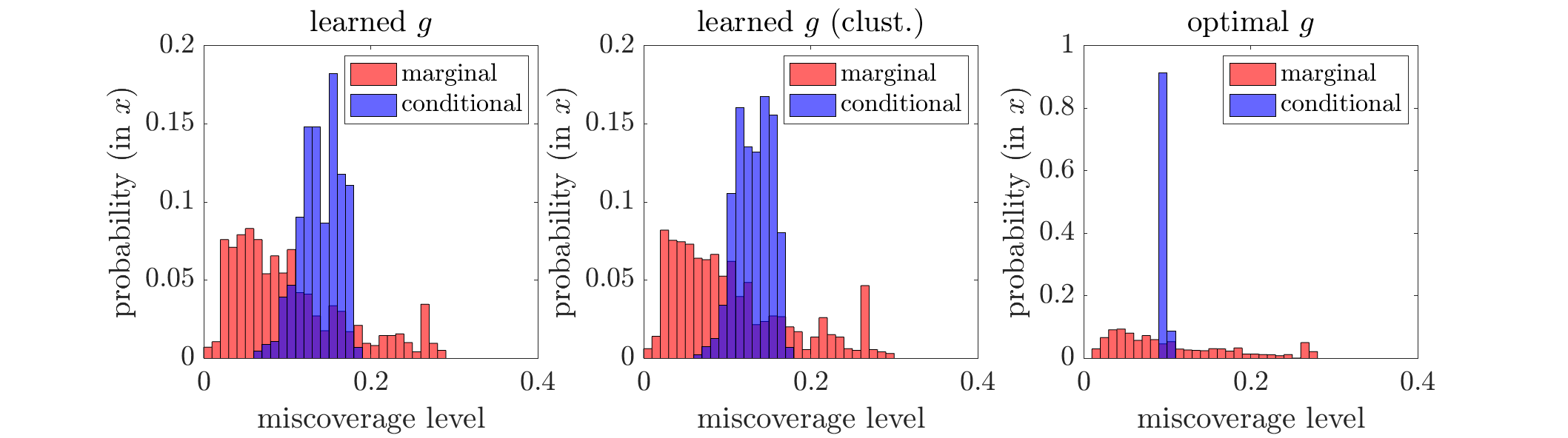}
     
     \vspace*{-.2cm}

     \caption{Comparing conditional coverage probabilities at a fixed level $\alpha = 0.1$, obtained from the marginal coverage formulation (a single $\lambda$) or with a conditional coverage formulation with an $\alpha$-dependent $\lambda$ described in \mysec{conformal}. Left: learned function $g$, middle: learned function $g$ with post-clustering, right: optimal function $g$.\label{fig:coverage}
     }
\end{figure}

\EIT

\subsection{Regression}
We consider synthetic experiments with univariate regression to assess the performance of the new loss function, with $\X = [0,1]$.

\BIT
\item \textbf{Comparison of modular and submodular penalties.} We compare in \myfig{cover_1d} the impact of using different submodular functions (modular one and set-covers), both with optimal estimation and learned ones with spline kernels. We can see the set-cover's flattening effect, which favors connected sets once thresholded.

\begin{figure}[h]
   \centering
     \includegraphics[width=0.65\linewidth]{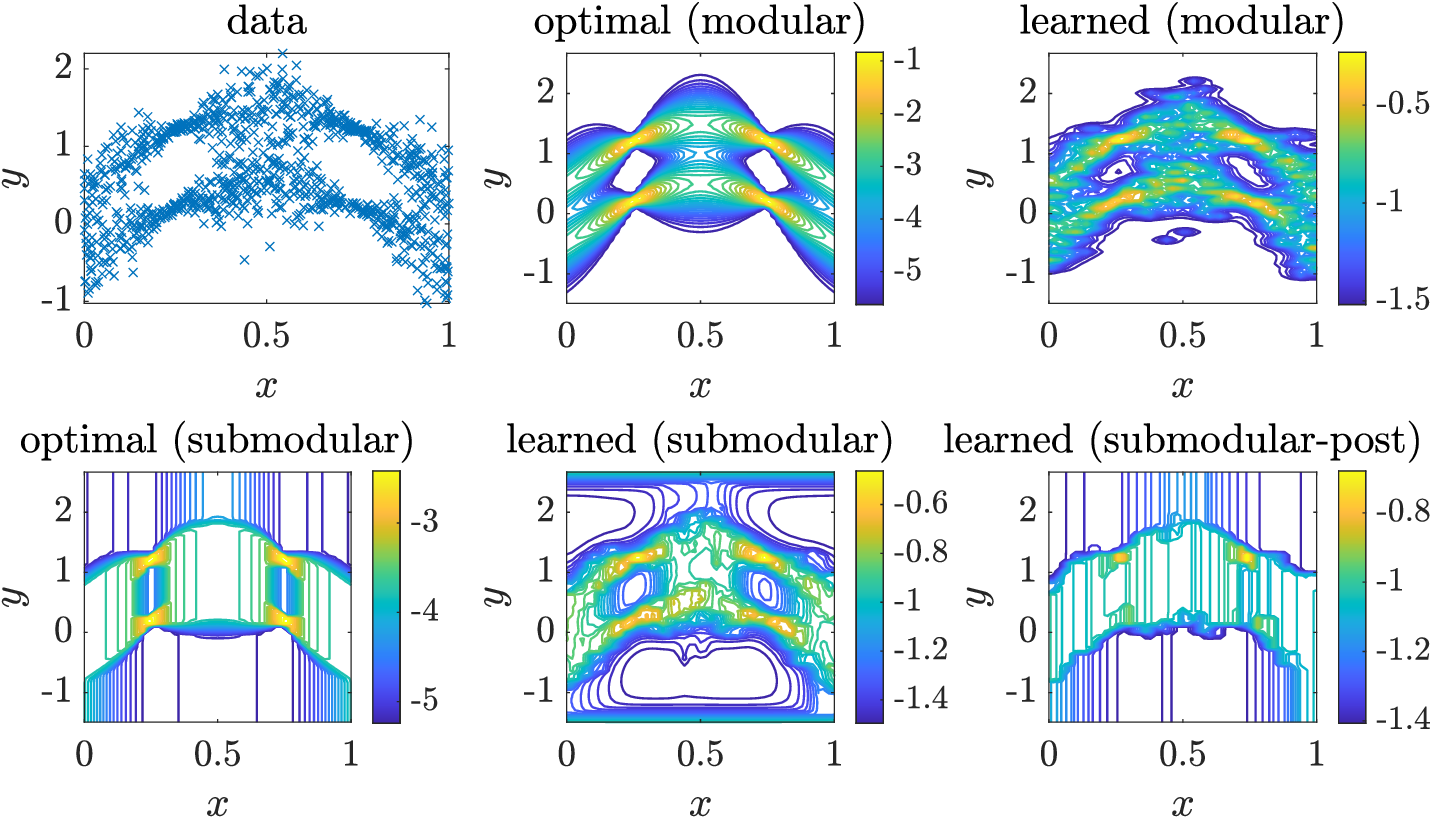}
     
     \vspace*{-.2cm}

     \caption{Comparing different estimated functions for a one-dimensional regression problem, for data from the top-left plot. Modular penalties are considered with the optimal prediction function in the top-middle plot, and the learned one in the top-right plot. In the bottom plots, we consider submodular penalties, with the optimal functions, and the learned functions, with and without post-clustering.\label{fig:cover_1d}
     }
\end{figure}

\item \textbf{Confidence sets.} We show in \myfig{cover_1d_coverage} for the simple modular $V$ the coverage sets for every $x \in \X $ at a given level $\alpha \in (0,1)$, for the strategy learned with our new loss function (both for a modular and a submodular penalty), as well a what can be learned by the interval loss precisely at this level~$\alpha$ (which is the sum of the pinball loss at levels $\alpha/2$ and $1-\alpha/2$), which can only output intervals. In contrast, our loss function leads to sets that are not intervals, with more regularity and fewer holes when considering the submodular penalty.
\EIT

\begin{figure}[h]
   \centering
     \includegraphics[width=0.75\linewidth]{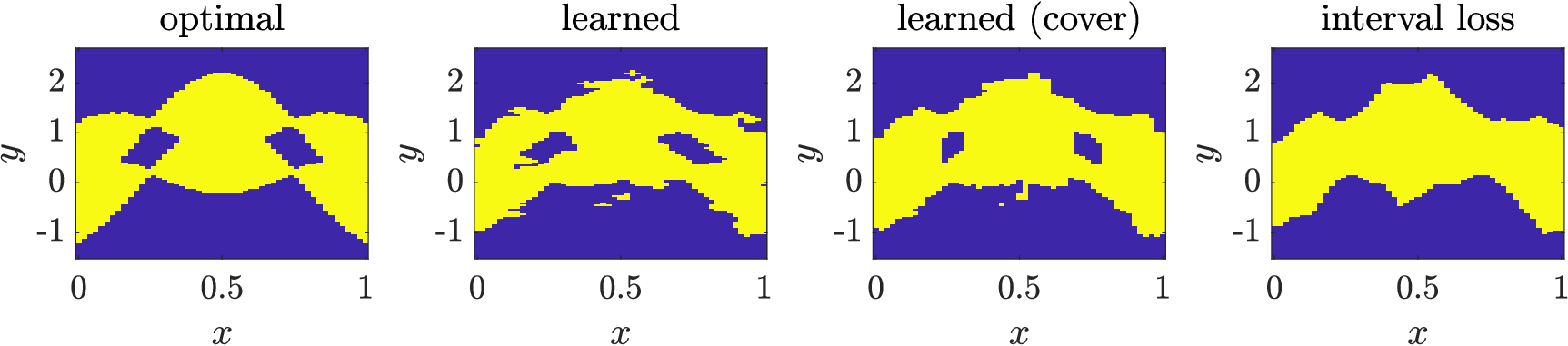}
     
     \vspace*{-.25cm}
     
     \caption{Comparing different estimated coverage at level $\alpha = 0.1$ for the one-dimensional regression problem from \myfig{cover_1d}. From left to right: optimal set, learned set with a modular penalty, learned set with cover penalty, intervall loss. \label{fig:cover_1d_coverage}
     }
\end{figure}

\section{Conclusion}
In this paper, we proposed a convex loss function for set prediction based on the Choquet integral. This allows the estimation of sets obtained by thresholding a real-valued function, leading to all potential conditional miscoverage levels, with an explicit trade-off between size and miscoverage. This work could be extended in several ways: (a) As described in \mysec{addex}, several relevant notions of size are not submodular, and the convexity of the loss is not satisfied anymore, and thus new formulations need to be defined.  (b) A detailed study of conditional coverage based on linear models could confirm the need for the extra regularizations that are proposed in \mysec{algorithms}, and highlight their provable benefits when used within the conformal prediction paradigm.
(c) Our loss functions have attractive computational properties for large-scale setups (e.g., quadratic, some form of separability, availability of stochastic gradients), which could be explored further.

\subsection*{Acknowledgements}
   The author thanks Eug\`ene Berta, Sacha Braun, David Holzm\"uller, and Michael Jordan, for insightful discussions related to this work.
   This work has received support from the French government, managed by the National Research Agency,
under the France 2030 program with the reference ``PR[AI]RIE-PSAI'' (ANR-23-IACL-0008).

\bibliography{submodularconformal}
  \end{document}